%% file: main.tex
\documentclass[twoside,11pt]{article}
\usepackage{jmlr2e}

\PassOptionsToPackage{numbers, compress}{natbib}
\usepackage[utf8]{inputenc} 
\usepackage[T1]{fontenc}
\usepackage{url}
\usepackage{hyperref}       % hyperlinks
\usepackage{booktabs}       % professional-quality tables
\usepackage{amsfonts}       % blackboard math symbols
\usepackage{nicefrac}       % compact symbols for 1/2, etc.
\usepackage{microtype}      % microtypography
\usepackage{ifthen}
\usepackage[linesnumbered,ruled,algo2e]{algorithm2e}
\usepackage{amsthm, amssymb, amsmath}
\usepackage{graphicx}
\usepackage{cleveref}
\usepackage[colorinlistoftodos,prependcaption]{todonotes}
\usepackage{epstopdf}
\usepackage{algorithm}
 \usepackage{algorithmic}
\usepackage{wrapfig}
\usepackage{subcaption}
\usepackage{bbm}

\input{GJ_macros}

\newcommand{\totalPulls}{T}
\newcommand{\setofArms}{\mathcal{C}}

\newcommand{\sampleSpace}{\mathcal{W}}
\newcommand{\numArms}{K}
\newcommand{\numPts}{J}
\newcommand{\arm}{k}
\newcommand{\meanReward}{\mu}
\newcommand{\Index}{I}
\newcommand{\pulls}{n}
\newcommand{\slot}{t}
\newcommand{\gap}{\Delta}
\newcommand{\estimateMean}{\hat{\phi}}
\newcommand{\expectedPseudoReward}{\phi}

\newcommand{\regret}{Reg}
\newcommand{\reward}{r}
\newcommand{\estimateReward}{s}
\newcommand{\optimistGap}{\tilde{\Delta}}
\newcommand{\indicator}{\mathbbm{1}}

\def \OO {\mathrm{O}}

\usepackage[utf8]{inputenc} % allow utf-8 input
\usepackage[T1]{fontenc}    % use 8-bit T1 fonts
\usepackage{hyperref}       % hyperlinks
\usepackage{url}            % simple URL typesetting
\usepackage{booktabs}       % professional-quality tables
\usepackage{amsfonts}       % blackboard math symbols
\usepackage{nicefrac}       % compact symbols for 1/2, etc.
\usepackage{microtype}      % microtypography

\usepackage{cleveref}
\pdfoutput = 1

\begin{document}
% \nipsfinalcopy is no longer used
\title{Correlated Multi-armed Bandits with a \\Latent Random Source}

\author{\name Samarth Gupta \email samarthg@andrew.cmu.edu \\
 \addr Carnegie Mellon University\\
 Pittsburgh, PA 15213 
 \AND
 \name Gauri Joshi \email gaurij@andrew.cmu.edu \\
 \addr Carnegie Mellon University\\
 Pittsburgh, PA 15213 
 \AND
 \name Osman Ya\u{g}an \email oyagan@andrew.cmu.edu\\
 \addr Carnegie Mellon University\\
 Pittsburgh, PA 15213}

\editor{No editors}
\maketitle

\begin{abstract}
%Multi-arm bandit algorithms are a widely applicable class of sequential decision-making algorithms that strike an exploration-exploitation trade-off.
We consider a novel multi-armed bandit framework where the rewards obtained by pulling the arms are functions of a common latent random variable. The correlation between arms due to the common random source can be used to design a generalized upper-confidence-bound (UCB) algorithm that identifies certain arms as \emph{non-competitive}, and avoids exploring them. As a result, we reduce a $K$-armed bandit problem to a $C+1$-armed problem, where $C+1$ includes the best arm and $C$ \emph{competitive} arms. Our regret analysis shows that the competitive arms need to be pulled $\OO(\log T)$ times, while the non-competitive arms are pulled only $\OO(1)$ times. As a result, there are regimes where our algorithm achieves a $\OO(1)$ regret as opposed to the typical logarithmic regret scaling of multi-armed bandit algorithms. We also evaluate lower bounds on the expected regret and prove that our correlated-UCB algorithm achieves $\OO(1)$ regret whenever possible.
%As a result of this dimensionality reduction, for certain reward functions, the regret is constant instead of logarithmically scaling with the number of sampling rounds $T$. Our regret analysis shows that this algorithm is order-wise optimal, that is, the competitive arms have to be pulled at least $\OO(\log T)$ times each in order to maximize the cumulative reward. 
\end{abstract}
\input{Introduction.tex}
\input{Problem_Formulation.tex}
\input{Algorithm.tex}

\input{regret.tex}

\input{Simulations.tex}
\input{Conclusion.tex}

\bibliographystyle{unsrt}
\bibliography{multi_armed_bandit}

%\newpage
\appendix

\input{appendix}

\end{document}

%% file: GJ_macros.tex
\newtheorem{lem}{Lemma}
\newtheorem{thm}{Theorem}
\newtheorem{defn}{Definition}

\newtheorem{exple}{Example}

\newtheorem{rem}{Remark}
\newtheorem{fact}{Fact}

\newcommand{\E}[1]{\mathbb{E}\left[{#1}\right]}

\DeclareMathOperator*{\argmax}{arg\,max}

\graphicspath{{Figures/}}

\crefname{equation}{}{}
\Crefname{equation}{}{}
\crefname{thm}{theorem}{theorems}
\Crefname{thm}{Theorem}{Theorems}
\crefname{clm}{claim}{claims}
\Crefname{clm}{Claim}{Claims}
\Crefname{coro}{Corollary}{Corollaries}
\Crefname{lem}{Lemma}{Lemmas}
\Crefname{sec}{Section}{Sections}
\crefname{app}{appendix}{appendices}
\Crefname{app}{Appendix}{Appendices}
\crefname{prop}{proposition}{propositions}
\Crefname{prop}{Proposition}{Propositions}
\Crefname{propty}{Property}{Properties}
\crefname{figure}{figure}{figures}
\Crefname{figure}{Figure}{Figures}
\crefname{fig}{figure}{figures}
\Crefname{fig}{Figure}{Figures}
\crefname{defn}{definition}{definitions}
\Crefname{defn}{Definition}{Definitions}
\crefname{fact}{fact}{facts}
\Crefname{fact}{Fact}{Facts}
\crefname{appendix}{appendix}{appendices}
\Crefname{appendix}{Appendix}{Appendices}
\crefname{algo}{algorithm}{algorithms}
\Crefname{algo}{Algorithm}{Algorithms}
\crefname{algorithm}{algorithm}{algorithms}
\Crefname{algorithm}{Algorithm}{Algorithms}
\crefname{conj}{conjecture}{conjectures}
\Crefname{conj}{Conjecture}{Conjectures}
\crefname{obs}{observation}{observations}
\Crefname{obs}{Observation}{Observations}

%% file: Introduction.tex
\section{Introduction}
\label{sec:introduction}

% \begin{itemize}
%     \item Intro - talk about simple stuff like exploration and exploitation
%     \item Main idea through an example with figure
%     \item Summary of results
%     \item Possible applications range from Medical Diagnosis to Advertisement
%     \item Story with Advertisement (If needed, Marketing example, can give similar story for medical diagnosis)
%     \item Related work - Relation of our model with contextual, global models
% \end{itemize}

\textbf{Multi-armed Bandits.} The \emph{multi-armed bandit} (MAB) framework is a special case of reinforcement learning \citep{sutton98reinforcement} where actions do not change the system state. At each time step we  obtain a reward by pulling one of $K$ arms which have unknown reward distributions, and the objective is to maximize the cumulative reward. The seminal work of Lai and Robbins \citep{lai1985asymptotically} proposed the upper confidence bound (UCB) arm-selection algorithm, and studied its fundamental limits in terms of bounds on \emph{regret}. Subsequently, multi-armed bandit algorithms \citep{bubeck2012regret, garivier2011kl} have been used in numerous applications including medical diagnosis \citep{villar2015multi}, system testing \citep{tekin2017multi}, scheduling in computing systems \citep{mora2009stochastic, krishnasamy2016regret, joshi2016efficient}, and web optimization \citep{white2012bandit, agarwal2009explore} among others. A drawback of the classical model is that it assumes independent rewards from the arms, which is typically not true in practice. 

\textbf{Related Work.} Motivated by this shortcoming, several variants of the multi-armed bandit framework have been proposed in recent years. A class of variants relevant to our work is contextual bandits \citep{zhou2015survey, agrawal2013thompson,agarwal2014taming,sakulkar2016stochastic, sen2017identifying}, where in each round we observe a contextual vector that provides side information about the reward of each arm. Instead of receiving side information, correlated multi-armed bandits exploit the inherent correlation between the rewards of arms arising due to a structural relationship between the arms, or a set of common parameters shared between them. Some recent works \citep{pandey2007multi-armed,wang2018regional,hoffman2014correlation, yahyaa2015correlated, srivastava2016correlated, mersereau2009structured, ata2015global,combes2017minimal} have studied the correlated multi-armed bandit problem. Many of these works consider specific types of correlation such as clusters of arms \citep{pandey2007multi-armed, wang2018regional} and Gaussian or invertible reward functions \citep{ata2015global} that depend on a constant hidden parameter vector $\mathbf{\theta}$ \citep{yahyaa2015correlated, ata2015global, combes2017minimal, maillard2014latent, lattimore2014bounded}. We consider latent \emph{random variable} $X$, instead of constant parameter $\mathbf{\theta}$. Some recent papers \citep{bresler2014latent} study the regret of such latent source models for collaborative filtering, with rewards belonging to the set $\{-1, 0, +1\}$. Instead of maximizing regret, \citep{gupta2018active} considers the same model as this paper, but with the objective of learning the distribution of the latent random variable $X$.

\textbf{Main Contributions.} We consider a novel correlated multi-armed bandit model with a latent random source $X$, and we allow the rewards to be arbitrary functions of $X$, as described in \Cref{sec:problem}. In \Cref{sec:algorithm}, we propose the C-UCB algorithm, which is a fundamental generalization of the classic UCB algorithm. The $C$-UCB algorithm uses observed rewards to generate \emph{pseudo-reward} estimates of other arms, and restricts the exploration to the arms that are deemed (empirically) competitive. Regret analysis in \Cref{sec:regret} shows that after $T$ rounds of sampling, the C-UCB algorithm achieves an expected regret of $C \cdot \OO(\log T) + \OO(1)$, where $C \in \{0, \ldots, K-1\}$ denotes the number of arms that are \emph{competitive} with respect to the optimal arm. Thus, when the correlation between the rewards results in $C$ being equal to $0$, C-UCB achieves constant regret scaling with $T$, which is an order-wise improvement over standard bandit algorithms like UCB. We also find a lower bound on expected regret and show that the proposed algorithm achieves bounded regret whenever possible. Simulation results in \Cref{sec:simulations} show that our C-UCB algorithm outperforms the vanilla UCB algorithm that does not exploit the correlation between arms. %, even in the case when $C=K-1$ and no arms are non-competitive.

\textbf{Applications.} Unlike the classic MAB model that considers arms with independent rewards, our framework captures several  applications where the rewards of arms $k=1, \ldots, K$ depend on a common source of randomness. For example, the response to $K$ possible advertisements/products can depend on a latent variable $X$ that represents the social/economic condition of a customer. Similarly, the reward for using one of the $K$ possible encoding/routing strategies in a wireless communication network may depend on the current state $X$ of a time-varying channel.

Through controlled experiments or supervised learning approaches, we can learn the reward function $g_k(\cdot)$ for each possible value of $X$. While it is possible to find the mappings $g_k(x)$ for a small control group with different $x$'s, learning the distribution $F_X$ of a large population is likely to be difficult and costly; e.g., imagine a company willing to expand to a new region/country with an unknown demographic, and trying to identify the best products/ads. Similarly, in a communication network, it may not be efficient/possible to obtain the channel state information at every node and at every time instant. In this setting, our framework will help obtain larger cumulative reward. In particular, instead of the correlation-agnostic MAB framework, our approach will leverage the previously learned correlations to reduce the regret. Also, unlike contextual bandits where a personalized recommendation is given after observing the context $x$, our framework identifies a single recommendation that appeals to a large population where these contexts are hidden. 

%In contrast, in our problem, the arms $Y_1, Y_2, \ldots Y_K$, are correlated through the common hidden variable $X$. Contextual bandits consider a context vector $x$ for each arm that governs its reward distribution, unlike a common hidden variable $X$ considered in the proposed research. The works on correlated multi-armed bandits consider specific types of correlation between arms, such as clusters of arms \citep{pandey2007multi-armed} and  Gaussian models \citep{yahyaa2015correlated}.

%\subsection{Related Work}
%\GJ{Rewrite this, and add more related work, in particular, global and regional bandits}

%\subsection{Main Contributions}
%\GJ{Summarize contributions and paper organization here}
%\GJ{Add a couple of sentences on applications of the proposed model.}
%\GJ{Maybe we can insert a table showing regret bounds, along with lower bound (atleast for asymptotic scenario).}

%% file: Problem_Formulation.tex
\section{Problem Formulation}
\label{sec:problem}

\begin{wrapfigure}{r}{0.35\textwidth}
\centerline{\includegraphics[width=5.0cm]{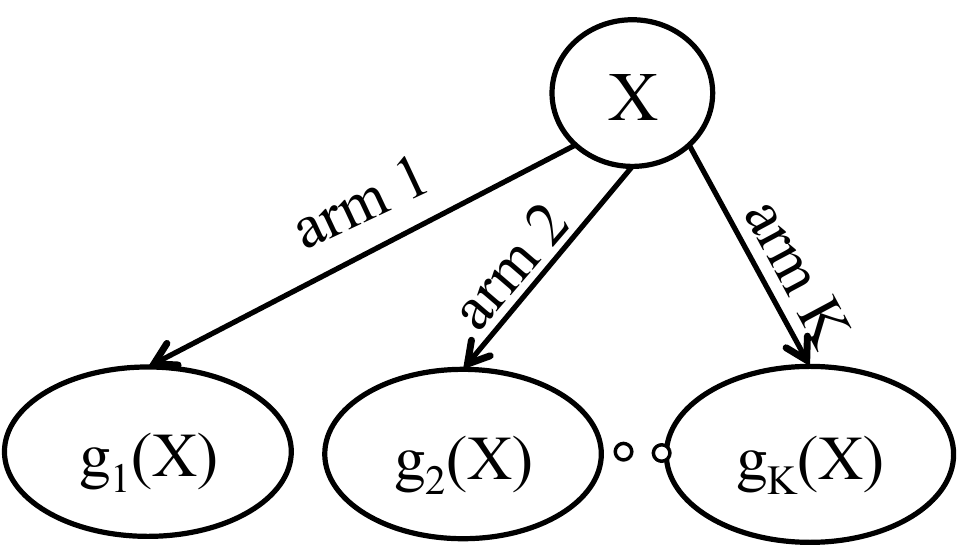}}
\caption{\small{The correlated multi-armed bandit framework. The reward of arm $k$ at round $t$ is $g_k(x_t)$, where $x_t$ is an i.i.d. realization of the latent random variable $X$. \label{fig:sys_model}} }
\vspace{-0.28 cm}
\end{wrapfigure}

\subsection{System Model and Regret Definition}

Consider a latent random variable $X$ whose probability distribution is unknown. The random variable can be either discrete or continuous. For discrete $X$, we denote the sample space by $\sampleSpace = \{x_1, x_2, \dots x_\numPts \}$, and use $p_j$ to denote the probability $\Pr(X=x_j)$ such that $\sum_{j=1}^{\numPts} p_j = 1$. For continuous $X$, $f_X(x)$ denotes the probability density function of $X$ over $x \in \mathbb{R}$.

Due to the latent nature of $X$, it is not possible to draw direct samples of $X$ and infer its unknown probability distribution. Instead, indirect samples can be obtained by choosing one of $\numArms$ \emph{arms} in each round $\slot$, where $\numArms$ is finite and fixed. Arm $\arm$ is associated with a reward function $g_\arm(X)$. If we take action $\arm_{\slot} \in \{1, 2 \ldots, \numArms\}$ in time slot $\slot$, we obtain the reward $g_{\arm_{\slot}}(x_{\slot})$ where $x_{\slot}$ is an i.i.d. realization of $X$ as shown in Figure~\ref{fig:sys_model}. The functions $g_1(X), g_2(X) \ldots g_\numArms(X)$ are assumed to be known.
Assume that there is a unique optimal arm $\arm^*$ that gives the maximum expected reward, that is,
\begin{equation}
\arm^* = \argmax_{\arm \in \{1,2, \ldots, \numArms\}}\E{g_\arm(X)} =  \arg \max_{\arm \in \{1,2, \dots \numArms\}} \meanReward_\arm,
\end{equation}
where $\meanReward_\arm$ denotes the mean reward of arm $\arm$. Let  $\Delta_\arm \triangleq \meanReward_{\arm^*} - \meanReward_\arm$ be defined as the sub-optimality gap of arm $\arm$ with respect to the optimal arm $\arm^*$. We also assume that the reward functions are bounded within an interval of size $B$, that is, $(\max_{x \in \mathcal{W}}g_\arm(x) - \min_{x \in \mathcal{W}}g_\arm(x)) \leq B$ for all arms $\arm \in \{1, \dots, \numArms\}$. 
%$\left| g_\arm(X) \right| \leq B$ for all possible values of the random variables $X$ and all arms $\arm \in \{1, \dots, \numArms\}$. 
%Without loss of generality, we normalize all reward functions such that $B = 1$. 
%
We do not make any other assumptions such as the functions $g_1, \dots g_\numArms$ being invertible. And indeed our problem framework and algorithm is most interesting when the reward functions are not invertible.

%\GJ{Need to formally express the fact that arms are pulled sequentially and that we can use past observations to choose the next arm}
% \SG{Our analysis assumes that reward is bounded between [0,1]. The results however can be easily generalized to any bounded rewards case.(will affect constants and conditions on $\gap_\arm$. Our algorithm does not need any condition on reward being bounded -- FIXED}
% \SG{Maybe we can mention here bounded rewards. In the analysis we can state that for analysis purposes we assume reward bounded between 0 and 1-- FIXED}
%\GJ{Yes, let us state that the rewards are bounded. Also, we should say that all the functions are non-negative.}
%\GJ{Make sure the regret definition is correct. Should it be just cumulative regret or expected regret? Need to check in other MAB papers}
Our objective is to sequentially pull 
arms $k_1, \ldots, k_t$ in order to maximize the cumulative reward. After $\totalPulls$ rounds, the cumulative reward is $\sum_{\slot = 1}^{\totalPulls} g_{\arm_\slot}(x_\slot)$.
%
%over $\totalPulls$ pulls; i.e., we wish to choose $k_1, \ldots, k_t$ such that 
% $$\lim_{\totalPulls \rightarrow \infty} G_\totalPulls = \lim_{\totalPulls \rightarrow \infty} \sum_{\slot = 1}^{\infty} \E{g_{i_\slot}(X_\slot)}.$$ 
%$$ G(\totalPulls) =  \sum_{\slot = 1}^{\totalPulls} \E{g_{\arm_\slot}(X_\slot)}$$ 
%is maximized.
%
%\SG{Mention that we do not need to know $T$ beforehand}
Maximizing the cumulative reward is equivalent to minimizing the cumulative regret which is defined as follows.
\begin{defn}[Cumulative Regret]
\label{defn:cumulative_regret}
The cumulative regret $\regret(\totalPulls)$ after $\totalPulls$ rounds is defined as
\begin{align}
    \regret(\totalPulls) &\triangleq \sum_{\slot = 1}^{\totalPulls} (g_{{\arm^*}}(x_\slot) - g_{\arm_\slot}(x_\slot))
   %= \sum_{\slot = 1}^{\totalPulls} \sum_{k=1}^{K} \indicator_{\{\arm_\slot = \arm\}} (g_{\arm}(x_\slot) - g_{{\arm^*}}(x_\slot))
 %\Delta_{\arm_\slot} 
   % = \sum_{\arm = 1}^{\numArms} \E{\pulls_\arm(T)} \Delta_\arm, 
    \label{eqn:regret}
\end{align}
where $x_t$ is an i.i.d.\ realization of $X$ that is not directly observed; we only observe $g_{k_t}(x_t)$. 
%\GJ{Should this definition be without the expectation? Then the algorithm is trying the minimize cumulative regret, and in the regret bounds we will take the expectation}
%\SG{This definition is actually for expected cumulative regret.(as we have expected number of pulls in rhs) I will take a look at other MAB papers and see if they refer to regret as an expectation quantity or not. }
%\GJ{Let us change the definition of regret to the one without expectation. The bounds can then be for $\E{\regret(T)}$}

%where, $\Delta_\arm \triangleq \E{g_\arm(X)} - \E{g_{\arm^*}(X)}$, the sub-optimality gap of arm $\arm$ with respect to the optimal arm $\arm^*$, and $\pulls_\arm(T)$ is the number of times arm $\arm$ is pulled in $\totalPulls$ slots.
\end{defn}

Thus, our goal is to design an algorithm to choose an arm $\arm_{\slot}$ at every round $\slot$ so as to minimize expected $\regret(\totalPulls)$. Note that we do not know the number of rounds $\totalPulls$ beforehand, and aim to minimize $\regret(\totalPulls)$ for all $\totalPulls$.

%Minimizing $\regret(T)$ is in turn equivalent to minimizing $\E{\pulls_\arm(T)}$, for each sub-optimal arm $k \neq k^*$. Thus, we seek to identify sub-optimal arms and minimize the number of pulls they are pulled. 

\begin{rem}[Connection to Classical Multi-armed Bandits]
\normalfont Although we consider a scalar random variable $X$ for brevity, our framework and algorithm can be generalized to a latent random vector $\mathbf{X} = (X_1, X_2, \dots X_m)$, as we explain in the supplementary material. The classical multi-armed bandit framework with independent arms is a special case of this generalized model when $\mathbf{X} = (X_1, X_2, \ldots X_\numArms)$ where $X_i$ are independent random variables and $g_\arm(X) = X_\arm$ for $k \in \{ 1, 2, \ldots, \numArms\}$.
\end{rem}

\subsection{Utilizing Correlation Between the Arms: Intuition and Examples}
In the classical multi-armed bandit framework there is a trade-off between exploring more arms to improve the estimates of their rewards, and exploiting the current best arm in order to maximize the cumulative reward. The sub-optimal arms have to be pulled $\Theta(\log T)$ times each, resulting in a $\Theta(\log T)$ cumulative regret as shown in the seminal work \citep{lai1985asymptotically}. In our new framework, since the reward functions $g_1, \dots g_\numArms$ are correlated through the common hidden random variable $X$, pulling one arm can give information about the distribution of $X$, which in turn can help estimate the reward from other arms. These \emph{pseudo-rewards} (defined formally in \Cref{sec:algorithm}) can allow us to declare certain arms as \emph{non-competitive} (defined formally in \Cref{sec:algorithm}) and pull them only $\OO(1)$ times. As a result, a $K$-armed bandit problem is reduced to a $C+1$-armed bandit problem, where $C \in \{0, 1, \ldots, \numArms-1 \}$ is the number of \emph{competitive} arms. Let us consider some examples to gain intuition on how arms are deemed non-competitive.

\begin{exple}[All Reward Functions are Invertible]
\normalfont Suppose that all the reward functions $g_1, \dots g_\numArms$ are invertible. Then, if we obtain a reward $r$ by pulling arm $\arm$ in slot $t$, it can be mapped back to a unique realization $x =  g_{\arm}^{-1}(r)$ of the latent random variable $X$. Using this realization, we can generate pseudo-samples $g_\ell(x)$ from any other arm $\ell \neq \arm$. This renders all sub-optimal arms non-competitive and obviates the need to explore them. As a result, a pure-exploitation strategy is optimal and it gives $\OO(1)$ regret.
\label{exple:invertible}
\end{exple}

\begin{wrapfigure}{r}{0.35\textwidth}
\centerline{\includegraphics[width=5.0cm]{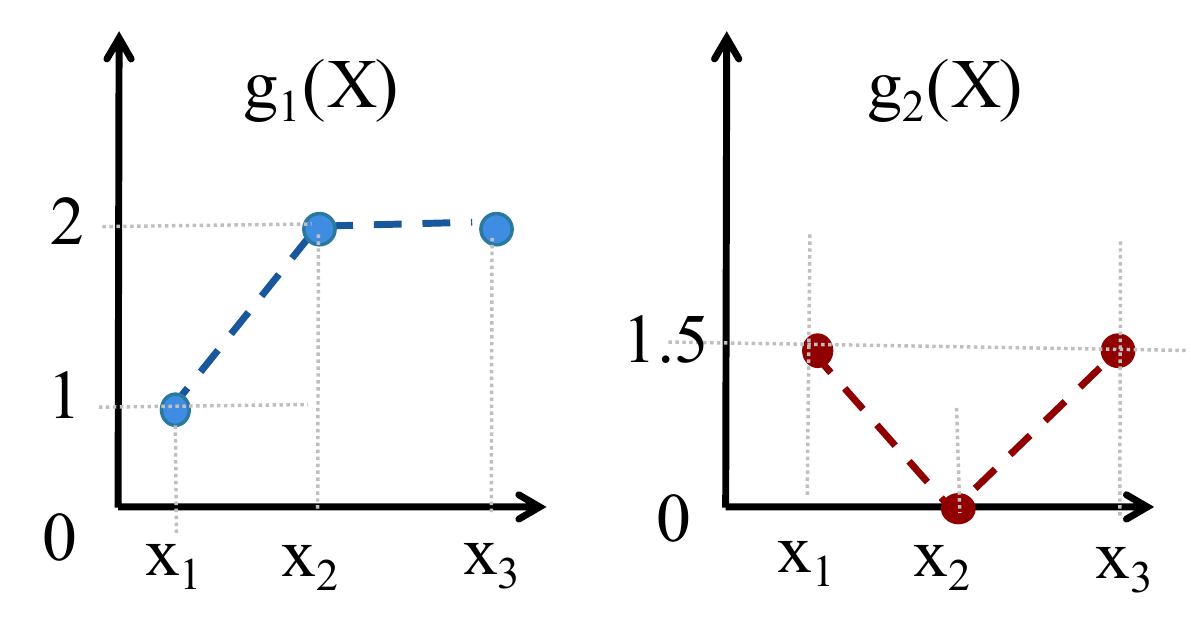}}
\caption{\small Example of two arms. \label{fig:two_arm_example} }
\vspace{-0.5 cm}
\end{wrapfigure}

In fact, it suffices to have only the function $g_{\arm^*}(x)$ corresponding to the optimal arm to be invertible to deem all other arms as non-competitive and to  achieve $\OO(1)$ regret; see Section~\ref{sec:regret} for details. To understand the intuition behind declaring arms as non-competitive for general reward functions, consider the two-arm example below. 

%observing the reward $g_k(X_t)$ from a particular arm $k$ will reveal the realization of $X_\slot$, which allows calculating the reward that could have been obtained  from  the other arms,  $g_{\ell}(X_\slot)$, $\ell \neq k$. In other words, no matter which arm is pulled at each time step, it will be possible to generate artificial samples  $g_k(X_t)$ for all arms at all time steps. 

\begin{exple}[Identifying Non-competitive Arms]
\normalfont Consider two-armed bandit problem with reward functions $g_1$ and $g_2$ respectively, as shown in \Cref{fig:two_arm_example}.
Suppose arm $1$ is pulled $10$ times, out of which we observe reward $1$ three times, and $2$ seven times, such that the empirical reward is
\begin{align}
\hat{\meanReward}_1 &= \hat{p}_1 + 2(\hat{p}_2 + \hat{p}_3) = 1.7 \label{eqn:mu_1}
\end{align}
Using \eqref{eqn:mu_1}, we can estimate the distribution $(p_1, p_2, p_3)$ of $X$ to be $\hat{p}_1 = 0.3$ and $\hat{p}_2 + \hat{p}_3$ = 0.7. It is not possible to use this to estimate the reward of arm $2$ since we only know the sum $\hat{p}_2 + \hat{p}_3$. However, we can find an upper bound on the empirical reward of arm $2$ as follows.
\begin{align}
\hat{\meanReward}_2 &= 1.5 \hat{p}_1 + 0   \hat{p}_2 + 1.5 \hat{p}_3 \\
&\leq 1.5 \hat{p}_1 + \max(0,1.5) (\hat{p}_2 + \hat{p}_3) = 1.5
\end{align}
Since the upper bound on arm $2$'s reward (which we refer to as its pseudo-reward) is less than arm $1$'s empirical reward, we consider arm $2$ as \emph{empirically non-competitive} with respect to arm $1$ and do not pull it until it becomes \emph{empirically competitive} again.
\label{exple:twoarms}
\end{exple}
%\SG{I believe we decided to use Competitive and Non-Competitive in an Expected sense. Using Competitive here in an empirical sense may become confusing}
% \SG{so that the "empirical reward is" (instead of "mean reward is"). "We can find an upper bound on the empirical reward of arm 2", Instead of "we can find an upper bound on the reward". -- FIXED}
% \GJ{Yes, that is true. Let us rewrite this example. Maybe we can use a simpler set of functions, and show the different regions where arm $2$ becomes non-competitive.}

% \GJ{I don't remember the simpler example that we discussed yesterday..}
% \SG{I think that is no longer needed as we have now defined the concept of empirical competitiveness}
% \SG{The simple example is Arm 1: $g_1(x_1) = 1, g_2(x_2) = 0, g_3(x_3) = 0$ and $g_2(x_1) = 0, g_2(x_2) = 0, g_2(x_3) = 1$. This maybe too simple an example though..}

%\GJ{Make sure that the symbols used for the estimated mean here is consistent with the rest of the paper}

In Section~\ref{sec:algorithm} below we formalize the idea of competitive and non-competitive arms and propose a correlated upper confidence bound (C-UCB) algorithm. In \Cref{sec:regret} we give upper and lower bounds on the regret of the proposed algorithm, and show that the regret is similar to that of UCB with just $C+1$ arms instead of $K$ arms, where $C$ is the number of competitive arms.% Simulation results show that our correlated UCB algorithm outperforms a vanilla UCB algorithm that does not exploit the correlation between arms. 

%The problem becomes more interesting when the arms are {\em non-invertible}, i.e., when it is not always the case that knowing $g_k(X_t)$ uniquely specifies $X_t$. In such cases, we only observe partial information about $X_\slot$. In the rest of the paper, we demonstrate how this partial information can be used to significantly reduce the number of times a suboptimal arm is pulled in a total of $\totalPulls$ rounds, thereby reducing the cumulative regret.  

%% file: Algorithm.tex
\section{C-UCB: The Proposed Correlated-UCB Algorithm}
\label{sec:algorithm}

Our algorithm to choose an arm in each round in the correlated multi-armed bandit framework is a fundamental generalization of the upper confidence bound (UCB1) algorithm presented in \citep{auer2002finite}. In round $\slot$, the UCB1 algorithm chooses the arm that maximizes the upper confidence index $\Index_\arm(t)$ which is defined as
\begin{align}
\Index_\arm(t) &= \hat{\meanReward}_\arm(t) + B\sqrt{\frac{2 \log \slot}{\pulls_\arm(\slot)}}, \label{eqn:UCB1_index}
\end{align}
where $\hat{\meanReward}_\arm(t)$ is the empirical mean of the rewards received from arm $\arm$ until round $t$, and $\pulls_\arm(\slot)$ is the number of times arm $\arm$ is pulled till round $\slot$. The second term causes the algorithm to explore arms that have been pulled only a few times (small $\pulls_\arm(\slot)$). Recall that we assume all rewards to be bounded within an interval of size $B$. When the index $t$ is implied by context, we abbreviate $\hat{\meanReward}_\arm(t)$ and $ \Index_\arm(t)$ to $\hat{\meanReward}_\arm$ and $\Index_\arm$ respectively in the rest of the paper. Also, we use the terms UCB1, UCB, and classic UCB interchangeably to refer to the UCB1 algorithm proposed in \citep{auer2002finite}. 

In correlated MAB framework, the rewards observed from one arm can help estimate the rewards from other arms. Our key idea is to use this information to reduce the amount of exploration required. We do so by evaluating the \emph{empirical pseudo-reward} of every other arm $\ell$ with respect to an arm $\arm$, as we saw in \Cref{exple:twoarms}. If this pseudo-reward is smaller than empirical reward of arm $\arm$, then arm $\ell$ is considered to be \textit{empirically non-competitive} with respect to arm $\arm$, and we do not consider it as a candidate in the UCB1 algorithm. 

The notions of pseudo-reward and empirical competitiveness of arms are defined in \Cref{sec:pseudo_reward} and \Cref{sec:competitive} below, and in \Cref{sec:modified_ucb} we describe how we modify the UCB1 algorithm. The pseudo-code of our algorithm is presented in \Cref{alg:formalAlgo}.

\subsection{Pseudo-Reward of Arm $\ell$ with respect to Arm $\arm$}
\label{sec:pseudo_reward}
%Due to correlation between the arms, when observing the reward from one arm $k$, we gain information about possible reward that we would have obtained on pulling some other arm $\arm'$. For example, in Example \ref{exple:twoarms}, if we observe a reward of $1$ from first arm, we can infer that reward obtained on pulling arm $2$ would have been $1.5$. Similarly, if we observe a reward of $2$ from first arm, we can infer that $X$ took the realization $x_2$ or $x_3$ and subsequently deduce that reward from arm 2 would have been either $1.5$ or $0$.

%We define \textit{Pseudo-Reward} of arm $\arm'$ with respect to arm $\arm$ as the maximum possible reward we would have obtained from arm $\arm'$ given that we received reward $r$ from the pulled arm $\arm$. For example, if we received a reward of $2$ on pulling arm $1$ in Example \ref{exple:twoarms}, then the Pseudo-Reward of arm $2$ with respect to arm $1$ for this round is $1.5$. We now define \textit{Pseudo-Reward} formally. 

%\GJ{The above explanation is a repetition of Example 2, so I think we should comment it out to save space.}

The pseudo-reward of arm $\ell$ with respect to arm $\arm$ is an artificial sample of arm $\ell$'s reward generated using the reward observed from arm $\arm$. It is defined as follows.

\begin{defn}[Pseudo-Reward]
\label{defn:pseudo_reward}
Suppose we pull arm $\arm$ and observe reward $\reward$. %Let $\preImageSet_{\arm}(\reward)$ be the set of realizations of $X$ that could have generated the reward $\reward$ from arm $\arm$, that is, all $x \in \sampleSpace$ that satisfy $g_\arm(x) = \reward$, where $\sampleSpace$ is the sample space of $X$. 
Then the pseudo-reward of arm $\ell$ with respect to arm $\arm$ is
\begin{align}
\estimateReward_{\ell,\arm}(\reward) \triangleq \max_{x: g_\arm(x)=r} g_{\ell}(x). \label{eqn:pseudo_reward}
\end{align}
\end{defn}

The pseudo-reward $\estimateReward_{\ell,\arm}(r)$ gives the maximum possible reward that could have been obtained from arm $\ell$, given the reward observed from arm $\arm$.
%\OY{Shall we add a note here saying if $g_k$ is invertible then, $\estimateReward_{\ell,\arm}(\reward) = g_{\ell}(g_k^{-1}(r))$ -- I think the example below leaves no room for confusion so pls ignore this comment}
In \Cref{exple:twoarms}, if we observe a reward of $r=2$ from arm $1$, $X$ could have been either $x_2$ or $x_3$. Then the pseudo-reward of arm $2$ is $\estimateReward_{2,1} = 1.5$ which is the maximum of $g_2(x_2)$ and $g_2(x_3)$. The pseudo-reward definition also applies to continuous $X$, and it can be directly extended to a latent random vector $\mathbf{X} = (X_1, \dots X_m)$ as well as explained in the supplementary material.

\begin{defn}[Empirical and Expected Pseudo-Reward]
\label{defn:empirical_pseudo_reward}
After $\slot$ rounds, arm $\arm$ is pulled $\pulls_\arm(\slot)$ times. Using these $\pulls_\arm(\slot)$ reward realizations, we can construct the empirical pseudo-reward $\estimateMean_{\ell, \arm}(\slot)$ for each arm $\ell$ with respect to arm $\arm$ as follows. 
\begin{align}
\estimateMean_{\ell, \arm}(\slot) \triangleq \frac{\sum_{\tau=1}^{\slot} \indicator_{k_\tau = k} \estimateReward_{\ell, \arm}(\reward_\slot)}{\pulls_\arm(\slot)}, \qquad \ell \in \{1,\ldots, K\} \setminus \{k\}. %, ~~\ell \neq k.
\end{align}
The expected pseudo-reward of arm $\ell$ with respect to arm $\arm$ is defined as
\begin{align}
\expectedPseudoReward_{\ell, \arm} \triangleq \E{\estimateReward_{\ell, \arm}(g_k(X))}.
\end{align}
%Moreover, $\reward_\slot= g_\arm(X_\slot)$ is a random variable, therefore we can define Expected Pseudo-Reward of arm $\arm'$ with respect to arm $\arm$ as $\E{\estimateReward_{\arm'}(\arm, r)}.$ Furthermore, Empirical Pseudo-Reward of arm $\arm'$ with respect to arm $\arm$ (over $\totalPulls$ rounds) is given by  $\estimateMean_{\arm'}(\arm) = \frac{\sum_{t:\arm_\slot = \arm}\estimateReward_{\arm'}(\arm, \reward_\slot)}{\pulls_\arm(\totalPulls)}.$
\end{defn}
%If $\pulls_\arm(\totalPulls) = \OO(T)$, the empirical pseudo-reward $\estimateMean_{\ell, \arm}(\totalPulls)$ will converge to its expected value $\expectedPseudoReward_{\ell, \arm}$.

%\OY{Not sure if the convergence statement below should be made a bit more solid, e.g., by reminding that it follows from SLLN and the fact that reward realizations are i.i.d. Also, it looks a bit odd that this is part of the definition. And strictly speaking, convergence would require $n_k(T)$ to grow large which is not necessarily the case -- after all, our main point is that even if $T \to \infty$ some arms are pulled only $O(1)$ times}
%\GJ{That is a good point.. Let me try and rewrite this definition to make clear. Please check it now}
%\SG{We use empirical pseudo reward for every round $\slot$ and not just for final round $\totalPulls$. Therefore definition should be for any $\slot$.}
\begin{algorithm}[t]
\hrule 
\vspace{0.1in}
\begin{algorithmic}[1]
\STATE \textbf{Input:} Reward Functions $\{g_1, g_2 \ldots g_\numArms\}$ %, Total number of rounds, $T$. 
\STATE \textbf{Initialize:} $\pulls_\arm = 0, \Index_\arm = \infty$ for all $\arm \in \{1, 2, \dots \numArms\}$ 
\FOR{ each round $\slot$}
\STATE Find $\arm^{\text{max}} = \arg \max_\arm \pulls_\arm(t-1)$, the arm that has been pulled most times until round $t-1$
\STATE Initialize the empirically competitive set $\mathcal{A} = \{1, 2, \ldots, \numArms\} \setminus \{\arm^{\text{max}}\}.$
%\STATE $c_t = \argmin_{k} \tilde{E}(S(t),\hat{p}(t-1)|S(t-1))$
\FOR{$\arm \neq \arm^{\text{max}}$}
\IF {$\hat{\meanReward}_{\arm^{\text{max}}} > \estimateMean_{\arm, \arm^{\text{max}}}$} 
\STATE Remove arm $\arm$ from the empirically competitive set: $\mathcal{A} = \mathcal{A} \diagdown \{\arm\}$ 
\ENDIF
\ENDFOR
%\STATE $S(t) = \{S(t-1), c_t\}$
 \STATE Apply UCB1 over arms in $\mathcal{A} \cup \{\arm^{\text{max}}\} $ by pulling arm $\arm_\slot = \arg \max_{\arm \in \mathcal{A} \cup \{\arm^{\text{max}}\}} \Index_\arm(t-1)$
\STATE Receive reward $\reward_{\slot}$, and update $\pulls_{\arm_\slot} = \pulls_{\arm_\slot} + 1$
\STATE Update Empirical reward: %\GJ{Check this new empirical reward update}:
$\hat{\meanReward}_{\arm_\slot}(\slot) = \frac{\hat{\meanReward}_{\arm_\slot}(\slot-1)(\pulls_{\arm_\slot}(t)-1) + r_t }{\pulls_{\arm_\slot}(t)}$
\STATE Update the UCB Index: $\Index_{\arm_\slot}(t) = \hat{\meanReward}_{\arm_\slot} + B\sqrt{\frac{2 \log \slot}{\pulls_{\arm_\slot}}}$
\STATE Compute pseudo-rewards for all arms $\arm \neq \arm_\slot$: \,\, % with respect to arm %$\arm_\slot$
$\estimateReward_{\arm,\arm_\slot}( \reward_{\slot}) = \max_{x: g_{\arm_\slot}(x) =\reward_{\slot}} g_{\arm}(x).$
\STATE Update empirical pseudo-rewards for all $ \arm \neq \arm_\slot$:  $\estimateMean_{\arm, \arm_\slot}(\slot) = \sum_{\tau: \arm_\tau = \arm_\slot} \estimateReward_{\arm,\arm_\tau}( \reward_\tau)/\pulls_{\arm_\slot}$
\ENDFOR
\end{algorithmic}
\vspace{0.1in}
\hrule
\caption{C-UCB Correlated UCB Algorithm}
\label{alg:formalAlgo}
%\GJ{This algorithm is a bit crowded and hard-to-read}
\end{algorithm}

%Observe that Pseudo-Reward is an upper bound on $g_{\arm'}(x_\slot)$, where $x_\slot$ is the realization of the random variable $X$ in round $\slot$. For the case when arm $\arm$ is invertible, Pseudo-Reward of arm $\arm'$ with respect to arm $\arm$ is equal to $g_{\arm'}(x_\slot)$. 

Note that the empirical pseudo-reward $\estimateMean_{\ell, \arm}(\slot)$ is defined with respect to arm $\arm$ and it is only a function of the rewards observed by pulling $\arm$. It may be possible to get a more accurate estimate of arm $\ell$'s reward by combining the observations from all other arms. However, we consider this rough estimate, and it is sufficient to reduce $K$-armed bandit problem to a $C+1$ armed problem, as we show in \Cref{sec:regret}.

\subsection{Competitive and Non-competitive arms with respect to Arm $\arm$}
\label{sec:competitive}

Using the pseudo-reward estimates defined above, we can classify each arm $\ell \neq k$ as {\em competitive} or {\em non-competitive} with respect the arm $\arm$. To this end, we first define the notion of the pseudo-gap.

\begin{defn}[Pseudo-Gap]
The pseudo-gap $\optimistGap_{\ell,\arm}$ of arm $\ell$ with respect to arm $\arm$ is defined as 
\begin{equation}
\optimistGap_{\ell,\arm} \triangleq \meanReward_{\arm} - \expectedPseudoReward_{\ell, \arm}, \label{eqn:pseudo_gap_defn} 
\end{equation}
i.e., the difference between expected reward of arm $\arm$ and the expected pseudo-reward of arm $\ell$ with respect to arm $\arm$.
\end{defn}
%\GJ{Do we need the notation to be  $\tilde{\Delta}$ or will $\Delta_{\ell, \arm}$ be enough? It will look cleaner without the tilde. -- Update -- Let us keep the tilde for now. Maybe it will make this notation easier to distinguish from $\Delta_k$}
%\GJ{Should we formally prove that the expected pseudo-reward is an upper bound on $\meanReward_{\ell}$?}
From the definition of pseudo-reward, it follows that the expected pseudo-reward $\expectedPseudoReward_{\ell, \arm}$ is greater than or equal to the expected reward $\meanReward_{\ell}$ from arm $\ell$. Thus, a positive pseudo-gap $\optimistGap_{\ell,\arm} > 0$ indicates that it is possible to classify arm $\ell$ as sub-optimal
using only the rewards observed from arm $k$ (with {\em high} probability as the number of pulls for arm $k$ gets {\em large}); thus, arm
 $\ell$ needs not be explored. Such arms are called non-competitive, as we define below. 

\begin{defn}[Competitive and Non-Competitive arms]
An arm $\ell$ is said to be non-competitive if its pseudo-gap with respect to the optimal arm $\arm^*$ is positive, that is, $\optimistGap_{\ell,\arm^*} > 0$. Similarly, an arm $\ell$ is said to be competitive if $\optimistGap_{\ell,\arm^*} < 0$. The unique best arm $k^*$ has $\optimistGap_{\arm^*,\arm^*} =0$ and is not counted in the set of competitive arms.
\end{defn}

Since the distribution of $X$ is unknown, we can not find the pseudo-gap of each arm and thus have to resort to empirical estimates based on observed rewards. In our algorithm, we use a noisy notion of the competitiveness of an arm defined as follows. Note that since the optimal arm $k^*$ is also not known, empirical competitiveness of an arm $\ell$ is defined with respect to each of the other arms $\arm \neq \ell$. 

%\GJ{Just realized that we haven't defined $\mu_k$ and $\hat{\mu}_k$ until this point in the paper. Maybe we can add these definitions in the problem formulation section.}

\begin{defn}[Empirically Competitive and Non-Competitive arms]
An arm $\ell$ is said to be \lq\lq empirically non-competitive with respect to arm $\arm$ at round $t$" if its empirical pseudo-reward is less than the empirical reward of arm $\arm$, that is, $\hat{\meanReward}_\arm(\slot) - \estimateMean_{\ell,\arm}(\slot) > 0$. Similarly, an arm $\ell \neq \arm$ is deemed empirically competitive with respect to arm $\arm$ at round $t$, if $\hat{\meanReward}_\arm(\slot) - \estimateMean_{\ell,\arm}(\slot) \leq 0$.
\end{defn}

\subsection{Modified UCB1 Algorithm to Eliminate Non-Competitive Arms}
\label{sec:modified_ucb}

The central idea in our correlated UCB algorithm is that after pulling the optimal arm $k^*$ sufficiently large number of times, the non-competitive (and thus sub-optimal) arms can be classified as     empirically non-competitive with increasing confidence, and thus need not be explored. As a result, the non-competitive arms will only be pulled only $\OO(1)$ times. However, the competitive arms cannot be discerned as sub-optimal by just using the rewards observed from the optimal arm, and have to be explored $\Theta(\log \totalPulls)$ times each. Thus, we are able to reduce a $K$-armed bandit to a $C+1$-armed bandit problem, where $C$ is the number of competitive arms. 

%Since non-competitive arms have $\optimistGap_{\arm}(\arm^*)$, a positive empirical estimate of Pseudo-Gap for arm $\arm$ based on large number of samples of optimal arm $\arm^*$ will indicate sub-optimality of arm $\arm$. Therefore, using just samples of optimal arm, arm $\arm$ can be identified as suboptimal if it is \textit{Non-Competitive}. Therefore, we need not explore \textit{Non-Competitive} arms explicitly. This is the key idea in our proposed C-UCB algorithm. Note that \textit{Competitive} arms cannot be discerned as suboptimal by just observing samples of optimal arm.

Using this idea, our C-UCB algorithm proceeds as follows. After every round $\slot$, we maintain values for empirical reward, $\hat{\meanReward}_\arm(\slot)$, and the UCB1 index $\Index_\arm(\slot)$ for each arm $\arm$. These empirical estimates are based on the $\pulls_\arm(\slot)$ samples of rewards that have been observed for $\arm$ till round $\slot$. In addition to this, we maintain empirical pseudo-reward of arm $\ell$ with respect to arm $\arm$, $\estimateMean_{\ell,\arm}(\slot)$, for all pairs of arms  $(\ell, \arm)$. 
%
%\GJ{Do we really need the empirical pseudo-rewards for all pairs of arms? Sounds computationally expensive}
%\SG{Yes, we do need it. It is really not computationally expensive. Empirical pseudo reward of arm $\arm'$ with respect to arm $\arm$ is updated only when arm $\arm$ is pulled. At max $K$ arithmetic computations in every round.}
% \GJ{Can shift these steps to the high-level summary and merge with the steps written there. Then we can just have two subsections following the high-level summary: estimating pseudo-rewards, and finding empirically competitive arms}
% \SG{Please have a look at the merged 3.1. Writing decision: Keep merged 3.1 or 3.3 with (very small intuitions in 3.1).}
% \SG{My personal opinion is that keeping these steps here (and not in 3.1, In 3.1 we should have just a intuition.) has more clarity and possible lesser room for confusion, and clearly highlights that Empirical Pseudo Reward is a pairwise quantity. Having description in 3.3 also allows for clear description of notations used in the formal algorithm. Also allows our subsection to be small.}
% \GJ{Yes, that might be good. I try and write such that we can put the steps here in 3.3}
%
In each round $\slot$, the algorithm performs the following steps: 
\begin{enumerate}
    \item Select arm $\arm^{\textit{max}} = \arg \max_{\arm} \pulls_{\arm}(t-1)$, that has been pulled the most until round $t-1$.
    \item Identify the set $\mathcal{A}$ of arms that are empirically competitive with respect to arm $\arm^{\textit{max}}$.
    \item Pull the arm $\arm_{\slot} \in \{\mathcal{A} \cup \arm^{\textit{max}} \} $ with the highest UCB1 index $\Index_\arm(t-1)$ (defined in \eqref{eqn:UCB1_index}).
    %\OY{couple of issues here: the set $\setofArms$ does not contain $\arm^{\textit{max}}$, and I think at this step we should still include $\arm^{\textit{max}}$ as a candidate arm. A more minor issue is whether we should write $k_{t+1}$ instead of $k_t$}
    %\GJ{Fixed the first issue. If we want to change the index to $t+1$, we need to make sure that we do that consistently everywhere. To minimize inconsistency, I am changing the previous index to $t-1$ and keeping the current index $t$}
    \item Update the empirical pseudo-rewards $\estimateReward_{\ell, \arm_\slot}$ for all $\ell$, the empirical reward $\estimateMean_{\ell,\arm_\slot}(t)$, and the UCB1 indices of all arms based on the observed reward $\reward_{\slot}$.
\end{enumerate}

In step $1$, we choose the arm that has been pulled the most number of times because we have the maximum number of reward samples from this arm. Thus, it is likely to most accurately identify the non-competitive arms. This property enables the proposed algorithm to achieve an $\OO(1)$ regret contribution from non-competitive arms as we show in \Cref{sec:regret} below. 
%\GJ{Check the above paragraph and rewrite if necessary}
%A formal description of C-UCB is presented in \Cref{alg:formalAlgo}. 

% \color{red}
% In the normal UCB algorithm we choose the arm with the highest UCB index among all arms at every round. In order to exploit the correlation, we differ from UCB in the following manner. For each round,  
% we begin by choosing (\color{red} Note: choosing, not pulling \color{black}) an arm that has been pulled the maximum number of times so far. We can pull one of the $K$ arms in this round. For this particular round, we eliminate a subset of arms and choose the arm with the highest UCB index among the remaining arms. An arm is eliminated in a particular round if the optimistic estimate of the mean reward stored by the chosen arm is smaller than the empirical reward of the chosen arm. \color{red} Elimination is only for the round and not for all the remaining time horizon. The process of elimination happens every round with a fresh start \color{black}. 

% \color{red}
% The reason behind first choosing an arm that has been pulled the maximum number of times is to ensure that arms are eliminated with best confidence in that round as the mean rewards and estimated reward of chosen arm have least variance among all arms. 
% \color{black}

%% file: regret.tex
\section{Regret Analysis and Bounds}
\label{sec:regret}
\color{black}

We now characterize the performance of the C-UCB algorithm by analyzing the expected value of the cumulative regret (\Cref{defn:cumulative_regret}). The expected regret can be expressed as
\begin{align}
    \E{\regret(\totalPulls)} &=
   \sum_{\arm = 1}^{\numArms} \E{\pulls_\arm(T)} \Delta_\arm, 
    \label{eqn:exp_regret}
\end{align}
where $\Delta_\arm = \E{g_{\arm^*}(X)} - \E{g_\arm(X)} = \meanReward_{\arm^*} - \meanReward_\arm $ is the sub-optimality gap of arm $\arm$ with respect to the optimal arm $\arm^*$, and $\pulls_\arm(T)$ is the number of times arm $\arm$ is pulled in $\totalPulls$ slots.

For the regret analysis, we assume without loss of generality that the reward functions $g_\arm(X)$ satisfy $0 \leq g_\arm(X) \leq 1$ for all $\arm \in \{ 1, 2, \dots \numArms \}$. Note that the C-UCB algorithm does not require this condition on $g_\arm(X)$, and the regret analysis can also be generalized to any bounded reward functions.

\subsection{Instance-Dependent Bounds} %Upper Bounds on Expected Regret}
Most works on multi-armed bandits derive two types of bounds on expected regret: instance-dependent and worst case bounds, depending on whether or not the minimum sub-optimality gap $\gap_{\text{min}}$ goes to $0$ with the total number of rounds $\totalPulls$. Our instance-dependent bounds assume that the minimum gap $\Delta_{\text{min}} = \min_\arm \Delta_\arm$ remains strictly positive as the number of rounds $T \rightarrow \infty$, which is generally true in practice. Worst-case bounds are required when $\Delta_{\text{min}}$ can be arbitrarily small for large $T$. We derive both these bounds for the correlated-UCB algorithm. We use the standard Landau notation in the results, where all asymptotic statements are for large $T$. The proofs of all the results presented below are deferred to the supplement.

% Instance-dependent bounds assume that the minimum gap $\Delta_{\text{min}} = \min_\arm \Delta_\arm$ remains strictly positive as the number of rounds $T \rightarrow \infty$, which is generally true in practice. More formally, let us consider a constant $t_0$, which is the minimum $\tau \in \{3, 4, \ldots \}$ for which $\gap_{\text{min}} \geq 4\sqrt{\frac{\numArms \log \tau}{\tau}}$.

In order to bound $\E{\regret(\totalPulls)}$ in \eqref{eqn:exp_regret}, we can analyze the expected number of times sub-optimal arms are pulled, that is, $\E{\pulls_\arm(\totalPulls)}$, for all $\arm \neq \arm^*$. \Cref{thm:NonCompetitiveBound} and \Cref{thm:CompetitiveBound} below show that $\E{\pulls_\arm(\totalPulls)}$ scales as $O(1)$ and $O(\log T)$ for non-competitive and competitive arms respectively. Recall that a sub-optimal arm is said to be non-competitive if its pseudo-gap $\optimistGap_{\arm,\arm^*}>0$, and competitive otherwise.

\begin{thm}[Expected Pulls of a Non-competitive Arm]
\label{thm:NonCompetitiveBound}
%Suppose there exists a number $t_0$ is th = \min_{\tau \in \{2, 3, \ldots\}$ such that $\tau$ such that $\gap_{\text{min}} \geq 4\sqrt{\frac{\numArms \log \tau}{\tau}}, \tilde{\gap}_{\arm,\arm^*} \geq 2 \sqrt{\frac{2 \numArms \log \tau}{\tau}}$.

If the pseudo-gap $\optimistGap_{\arm,\arm^*} \geq 2\sqrt{\frac{2 \numArms \log \slot_0}{\slot_0}}$, and the sub-optimality gap  $\gap_{\text{min}} \geq 4\sqrt{\frac{\numArms \log \slot_0}{\slot_0}}$ for some constant $\slot_0 > 0$ then
\begin{align}
\E{\pulls_\arm(\totalPulls)} &\leq \numArms \slot_0 + \numArms (\numArms - 1) \sum_{\slot= \numArms \slot_0}^{\totalPulls} 3 \left(\frac{\slot}{\numArms}\right)^{-2} + \sum_{\slot = 1}^{\totalPulls} \slot^{-3}, \label{eqn:upper_bnd_comp}\\
&= \OO(1). %\quad \text{ for large } T.
\end{align} 
%\OY{Let's get rid of these "for large $T$" terms. In Introduction of somewhere earlier we can mention that we use the standard Landau notation and all asymptotic statements are made with respect to $T \to \infty$.}

%Consequently, if $\gap_{\text{min}} > \delta_1$ and $\optimistGap_{\arm, \arm^*} > \delta_2$, for any $\delta_1, \delta_2 > 0$, the expected number of times a \textit{Non-Competitive} arm is pulled is upper bounded by a constant. More formally, $$\lim_{\totalPulls \rightarrow \infty} \E{\pulls_{\arm}(\totalPulls)} \leq \alpha_\arm \quad \text{ for some } \alpha_k > 0, \text{ or,} \quad \E{\pulls_\arm(\totalPulls)} = \OO(1).$$ %\GJ{Put the O(1) inside the previous equation, in the same line}
\end{thm}
%\SG{How about this? Clearly explains that our conditions on $\gap_{\text{min}}, \optimistGap_{\arm, \arm^*}$ are quite mild.}

% \TODO{Theorem below assuming all gaps $\gap_\arm, \optimistGap_\arm(\arm^*)$ are some constants, so don't worry about corner cases. Corner cases are worst case regret bound presented next.}

%\GJ{The difference between the setting of Theorem 1 and 2 is unclear. I think Theorem 1 is for arms that are eliminated by UCB (need to give some name to such arms) and Theorem 2 is for arms that are sub-optimal but are not eliminated by our algorithm.}

%\GJ{Name of this theorem: Competitive, but Sub-optimal arms?}
%\SG{Actually the below is valid for any sub optimal arm. For Non-Competitive we have a stronger result in Theorem 1}
\begin{thm}[Expected Pulls of a Competitive Arm]
\label{thm:CompetitiveBound}
Expected number of times a competitive arm is pulled can be bounded as % sub-optimal arm $\arm$ is pulled is always upper bounded by a term that scales logarithmically in total number of pulls, More specifically,
\begin{align}
\E{\pulls_\arm(\totalPulls)} 
&\leq 8 \frac{\log (\totalPulls)}{\gap_\arm^2} + \left(1 + \frac{\pi^2}{3}\right) + \sum_{\slot = 1}^{\totalPulls} \slot \exp\left(- \frac{\slot \gap_{\text{min}}^2}{2 \numArms}\right), \label{eqn:upper_bnd_non_comp}\\
&= \OO(\log \totalPulls) \quad \text{ if  } \gap_{\text{min}} = \min_k \Delta_\arm  > 0.
\end{align}
%\OY{I think this last step also needs $\Delta_k >0$. I think it would be better to state these simply as if $\Delta_k,\Delta_{min}>0$}
%\GJ{Fixed}
%If the sub-optimality gap $\gap_{\text{min}} > \delta_1$ for some $\delta_1 > 0$, we have
%$$\lim_{\totalPulls \rightarrow \infty} \E{\pulls_\arm(\totalPulls)} = \beta_\arm \log (\totalPulls) \quad \text{for some } \beta_\arm > 0,\text{ or, }\quad \E{\pulls_\arm(\totalPulls)} = \OO(\log (\totalPulls).$$ %\GJ{Put the $\OO(\log \totalPulls)$ inside the previous equation, in the same line}
\end{thm}
% \SG{Because we do not make any assumptions above, the above result is valid for all suboptimal arms (even the non-competitive arms. So changed non-competitive to Any suboptimal arm.}

%\OY{I would delete this sentence starting with "note that ..."}
%Note that the bound in \Cref{thm:CompetitiveBound} above also holds for any non-competitive arm, but for non-competitive arms, the bound in \Cref{thm:NonCompetitiveBound} is tighter. 

Substituting the bounds on $\E{\pulls_\arm(\totalPulls)}$ derived in \Cref{thm:NonCompetitiveBound} and \Cref{thm:CompetitiveBound} into \eqref{eqn:exp_regret}, we get the following upper bound on expected regret.

\begin{thm}[Upper Bound on Expected Regret]
\label{thm:upper_bnd_exp_regret}
If the minimum sub-optimality gap $\gap_{\text{min}} \geq 4\sqrt{\frac{\numArms \log \slot_0}{\slot_0}}$, and the pseudo-gap of non-competitive arms $\optimistGap_{\arm,\arm^*} \geq 2\sqrt{\frac{2 \numArms \log \slot_0}{\slot_0}}$ for some constant $t_0 > 0$, then the expected cumulative regret of the C-UCB algorithm is
%logarithmically if $\gap_{\min} > \delta_1$ for some $\delta_1 > 0$. It has a multiplicative dependence on the number of \textit{Competitive} arms. More formally,
\begin{align}
\E{\regret(\totalPulls)} &\leq \sum_{\arm \in \setofArms} \Delta_\arm  U^{(c)}_\arm(\totalPulls) + \sum_{\arm' \in \{ 1, \ldots , \numArms \} \setminus \{ \setofArms \cup \arm^* \}
}\Delta_{\arm'} U^{(nc)}_{\arm'}(\totalPulls) , \label{eqn:upper_bnd_exp_regret}\\
&= C \cdot \OO(\log \totalPulls) + \OO(1), \label{eqn:upper_bnd_exp_regret_order}
%\quad \text{ for large } \totalPulls. 
\end{align} 
where $\mathcal{C} \subseteq \{ 1, \ldots , \numArms \} \setminus \{\arm^*\}$ is set of competitive arms with cardinality $C$, $U^{(c)}_\arm (\totalPulls)$ is the upper bound on $\E{\pulls_\arm(\totalPulls)}$ for competitive arms given in \eqref{eqn:upper_bnd_non_comp}, and $U^{(nc)}_\arm(\totalPulls)$ is the upper bound for non-competitive arms given in \eqref{eqn:upper_bnd_comp}.
%\OY{One suggestion here: We do not mention the conditions on $\Delta$'s with explicit $t_0$ (just does not make sense since $t_0$ is not in the result below or anywhere else in the main part of the paper). }

%$$\lim_{\totalPulls \rightarrow \infty} \E{\pulls_\arm(\totalPulls)} \leq \sum_{\arm : \arm \in \mathcal{C}} \beta_\arm \log(\totalPulls) + \alpha \quad \text{for some } \alpha, \beta_\arm > 0.$$
%Here, $\mathcal{C}$ is the set of competitive arms, i.e $\{\ell: \optimistGap_{\ell, \arm*} < \delta_2\}$ for any $\delta_2 > 0$. 
%\label{coro:cucbRegretBound}
\end{thm}

%\GJ{We are talking about this dimension reduction many times. Maybe we can omit this remark, or omit the dimension-reduction from some other places.}
\vspace{0.05cm}
\begin{rem}
\normalfont
If the set of competitive arms $\setofArms$ is empty (i.e., the number of competitive arms 
$C =0$), then our algorithm will lead to (see \eqref{eqn:upper_bnd_exp_regret_order}) an expected regret of $\OO(1)$, instead of the typical $\OO(\log T)$ regret scaling in classic multi-armed bandits. A simple case where $\setofArms$ is empty is when the reward function $g_{\arm^*}(X)$ corresponding to the arm $\arm^*$ is invertible. This is
because, for all sub-optimal arms $\ell \neq k^*$, the pseudo-gap $\optimistGap_{\ell, \arm^*} = \gap_\ell > 0$, resulting in those arms being non-competitive. The set $\setofArms$ can  be empty in more general cases where none of the arms are invertible. Then, our algorithm still achieves an expected regret of $\OO(1)$.
\end{rem}
%\SG{Arm is invertible needs to be clarified? It's corresponding function $g_{\arm^*}(x)$ is invertible. maybe in problem formulation when we say all reward function are invertible we can say that we refer to arm as invertible throughout this paper if its reward function is invertible?--FIXED}

%\OY{It may be a good idea to change the sentence above starting with "one scenario" a little bit. It is actually the most simple case and may be we should mention that our algo leads to O(1) even under more complicated cases. }
%\GJ{Fixed}

\vspace{0.1cm}
\begin{rem}
\normalfont
For the UCB1 algorithm \citep{auer2002finite}, the first sum in \eqref{eqn:upper_bnd_exp_regret} is taken over all arms. In this sense, our C-UCB algorithm is able to reduce a $\numArms$-armed bandit problem to a $C+1$-armed bandit problem. 
%\OY{We may refer to the simulation results here and say that experiments indicate we beat UCB1 even when $C \geq 1$.}
\end{rem}

Next, we present a lower bound on the expected regret $\E{\regret (\totalPulls)}$. Intuitively, if an arm $\ell$ is \textit{competitive}, it can not be deemed sub-optimal by only pulling the optimal arm $\arm^*$ infinitely many times. This indicates that exploration is necessary for competitive arms. The proof of this bound closely follows that of the 2-armed classical bandit problem \citep{lai1985asymptotically}; i.e., we construct a new bandit instance under which a previously sub-optimal arm becomes optimal without affecting reward distribution of any other arm. %This idea is quantified by the Lai-Robbins lower bound in our framework. 

\begin{thm}[Lower Bound on Expected Regret]
\label{thm:lower_bnd_exp_regret}

For any algorithm that achieves a sub-polynomial regret, %if for any arm $\arm$, we have $\optimistGap_{\arm,\arm^*} > \delta$, for some $\delta > 0$, i,e
%\mathbb{E}[Reg(T)] \geq \max_{k \in \mathcal{C}} \frac{\Delta_k}{D(f_{R_k} || f_{\tilde{R}_k})} \log(T).

\begin{equation}
    \lim_{\totalPulls \rightarrow \infty}\inf \frac{\E{\regret (\totalPulls)}}{\log (\totalPulls)} \geq 
    \begin{cases} 
    \max_{\arm \in \setofArms}\frac{\Delta_k}{D(f_{R_k} || f_{\tilde{R}_k})} \quad &\text{if } C > 0,\\
    0 \quad &\text{if } C = 0. 
    \end{cases}
\end{equation}

% \begin{align}
% \lim_{\totalPulls \rightarrow \infty}\inf \frac{\E{\regret (\totalPulls)}}{\log (\totalPulls)} \geq \max_{\arm \in \setofArms}\frac{\Delta_k}{D(f_{R_k} || f_{\tilde{R}_k})}.
% \end{align}
\label{thm:lowerBound}
 \end{thm}
Here $f_{R_k}$ is the reward distribution of arm $k$, which is linked with $f_X$ since $R_k = g_k(X)$. The term $f_{\tilde{R}_{k}}$ represents the reward distribution of arm $k$ in the new bandit instance where arm $k$ becomes optimal and distribution $f_{R_{k^{*}}}$ is unaffected. The divergence term represents "the amount of distortion needed in $f_X$ to make arm $k$ optimal", and hence captures the problem difficulty in the lower bound expression. 
%Moreover, we have $D(f_{R_k} || f_{\tilde{R}_k}) \geq D(f_{R_k} || f_{R_{k^*}})$ (due to correlation through $X$), which makes our lower bound stronger than the classical MAB lower bound for the 2 armed bandit problem with $C = 1$.

%\OY{The statement of this theorem is a bit weird. Let's talk about this during the meeting.}
%\SG{I have inserted the condition. We no longer require reward distribution of all arms to have the same support, only the reward distribution for "old" suboptimal and "converted" suboptimal need to be same, which is satisfied if $\exists \delta > 0$ for which, $\optimistGap_{\arm,\arm^*} > \delta$ }
%\SG{We define arm to be competitive if $\optimistGap_{\arm,\arm^*} > 0$, for our results to hold we need ($\optimistGap$ to not be arbitrarily close to zero), $\optimistGap_{\arm, \arm^*} > \delta$  Both are same things?}
%\OY{$\optimistGap_{\arm, \arm^*} > \delta$ for any $\delta > 0$ holds only if $\optimistGap_{\arm, \arm^*} \to \infty$}
%\SG{Right, I mean to write, $\exists \delta > 0$ for which $\optimistGap_{\arm,\arm^*} > \delta$. As long as this $\delta$ is positive and we satisfy this condition, we have the bound..}
%\OY{OK. If your result has this positive $\delta$ explicitly, then it makes sense to present it this way. Otherwise, you can just say 
%$\optimistGap_{\arm,\arm^*} >0$. This ensures that there are uncountably many $\delta>0$ values for which $\optimistGap_{\arm,\arm^*}>\delta$. If this $\delta$ appears only in the proof, I would  present the theorem without any $\delta$.}

\begin{rem}
\normalfont
From \Cref{thm:upper_bnd_exp_regret}, we see that whenever $C > 0$, our proposed algorithm achieves $\OO(\log \totalPulls)$ regret matching the lower bound given in Theorem \ref{thm:lowerBound} order-wise. Also, when $C = 0$, our algorithm achieves $\OO(1)$ regret. Thus, our algorithm achieves bounded regret whenever possible, i.e., when $C = 0$. 
\end{rem}

% \GJ{Add corollaries showing the overall regret for different cases, when all other arms are non-competitive, and when $C$ arms are competitive. Then we will also be showing the dimensionality reduction achieved by the proposed algorithm}

% \GJ{Another alternative is to put just one Theorem (before Theorem 1 and Theorem 2) and state the expected regret bound for $C$ competitive arms.}
% \SG{I like having two theorems and a corollory (for regret expression), It highlights our contributions clearly. If we have just one theorem, a reader may think that "oh regret is still logarithmic, no big deal".}
% \SG{I will add the corollaries for special cases}
% % \begin{defn}
% True optimistic reward of arm $\arm^*$ with respect to arm $\arm$, i.e $\optimistMean_{\arm^*}(\arm)$ is given by
% \begin{align*}
%  & \max_{\tilde{f}_X} \E{g_{\arm^*}(X)}  \\ 
% \text{s.t.  } &\tilde{f}_X(g_{\arm}(X)) = f_X(g_{\arm}(X))
% \end{align*}
% \end{defn}

%\GJ{In all of these theorems it is not clear what the constant $t_0$ is..}
%\SG{ minimum constant $\tau$ such that $\gap_{\text{min}} \geq 4\sqrt{\frac{\numArms \log \tau}{\tau}}, \tilde{\gap}_{\arm,\arm^*} \geq 2 \sqrt{\frac{2 \numArms \log \tau}{\tau}}$.}

\subsection{Worst Case Bound on Expected Regret}

Our instance-dependent bounds assumed that the minimum gap $\gap_{\text{min}} \geq 4 \sqrt{\frac{\numArms \log \slot_0}{\slot_0}}$ for some $\slot_0 > 0$, with a similar assumption on the pseudo-gap. We now present an upper bound the on expected regret without this assumption, when $\Delta_k$ can scale with $\totalPulls$ and become arbitrarily small as $\totalPulls \rightarrow \infty$.

%$\slot_0$ is comparable to $\totalPulls$, that is
%\GJ{Check if the above description of worst case regret is correct}

\begin{thm}[Worst Case Expected Regret]
In the worst case, the expected regret of the C-UCB algorithm is $\OO(\sqrt{\totalPulls \log (\totalPulls)})$.
\label{thm:worstCase}
\end{thm}

Note that this worst case regret bound is the same as that obtained for the UCB1 algorithm \citep{auer2002finite} when the arms are independent. This demonstrates that our algorithm can achieve the same order-wise worst case regret as classic UCB.

%% file: Simulations.tex
\section{Simulation Results}
\label{sec:simulations}

We now present simulation results for the case where $X$ is a discrete random variable (simulations for continuous $X$ and random vector $\mathbf{X}$ are shown in the supplement). %\GJ{Remove the random vector $\mathbf{X}$ from here if we don't end up putting simulation results for that case}
We consider the reward functions $g_1(X), g_2(X)$ and $g_3(X)$ shown in \Cref{fig:sim_reward_funcs} for all simulation plots. However, the probability distribution ${P}_X = (p_{x_1}, p_{x_2}, \dots p_{x_5})$ of $X$ is different for each of the following cases given below. For each case, \Cref{fig:simulationDiscrete} shows the cumulative regret versus the number of rounds. The cumulative regret is averaged over $500$ simulation runs, and for each run we use the same reward realizations for both the C-UCB and the vanilla UCB1 algorithms.

\begin{wrapfigure}{r}{0.5\textwidth}
\centerline{\includegraphics[width=7.2cm]{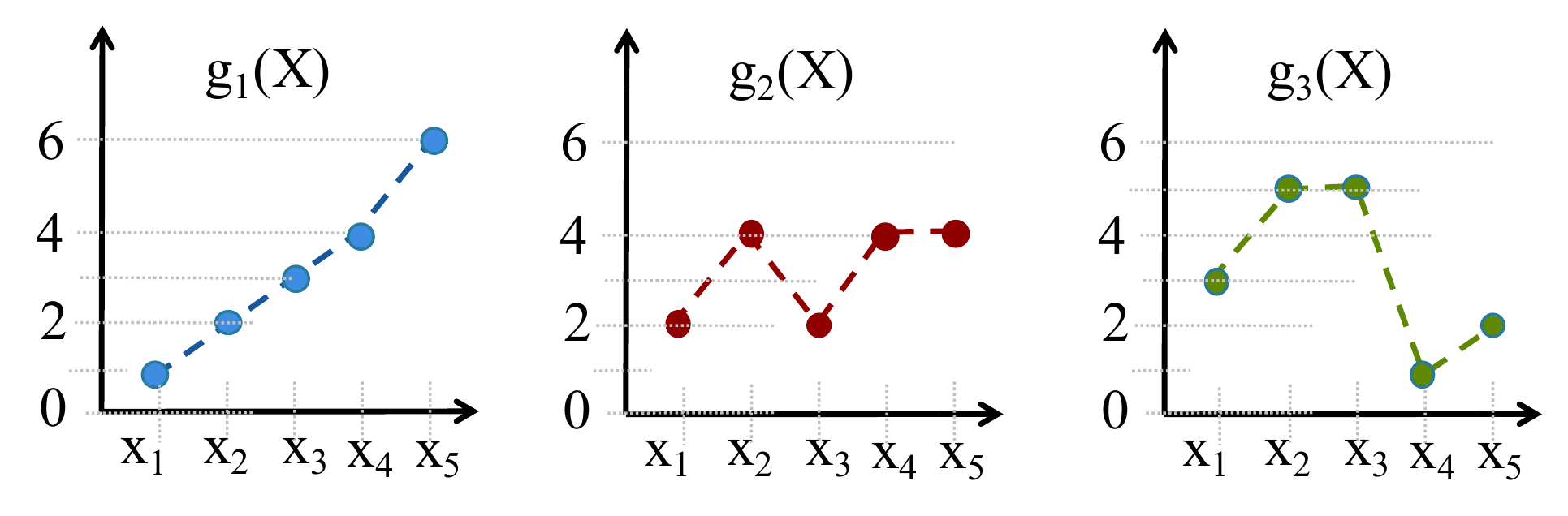}}
\caption{Reward Functions used for the simulation results presented in \Cref{fig:simulationDiscrete}. \label{fig:sim_reward_funcs} }
\vspace{-0.35cm}
\end{wrapfigure}

\textbf{Case 1: No competitive arms.}
Here, we set ${P}_X = (0.1, 0.2, 0.25, 0.25, 0.2)$. For this probability distribution, arm $1$ is optimal, and arms $2$ and $3$ are \textit{non-competitive}. Since both arm $2$ and arm $3$ are non-competitive, our result from \Cref{thm:NonCompetitiveBound} suggests that regret of C-UCB algorithm should not scale with the number of rounds $\totalPulls$. This is supported by our simulation results as well. We see in \Cref{fig:InvertibleDiscrete} that the proposed C-UCB algorithm achieves a constant regret and is significantly superior to the UCB1 algorithm as it is able to exploit the correlation of rewards between the arms. 

\textbf{Case 2: One competitive arm.} Let ${P}_X = (0.25, 0.17, 0.25, 0.17, 0.16)$ which results arm $3$ being optimal. Arm $1$ is \textit{non-competitive} while arm $2$ is \textit{competitive}. We expect from our results that number of pulls of arm $1$ should not scale with $\totalPulls$, while the number of pulls for arm $2$ can scale with the $\totalPulls$. This phenomenon can be seen in \Cref{fig:OneCompetitiveDiscrete}. The regret of C-UCB algorithm is much smaller than the UCB1 algorithm as C-UCB algorithm is not exploring arm 1. However, the regret scales with the number of rounds $\totalPulls$ as it is necessary to explore Arm 2. 

\textbf{Case 3: Two competitive arms.} In the last scenario, we set ${P}_X = (0.05, 0.3, 0.3, 0.05, 0.3)$. For this distribution, arm $3$ is optimal and arms $1$ and $2$ are both \textit{competitive}. Since both arms are competitive, exploration is necessary for both arms. Therefore, as we see in \Cref{fig:TwoCompetitiveDiscrete}, the regret obtained under C-UCB and UCB1 are similar and scale with the number of rounds $\totalPulls$.
%NoCompetitive.png, OneCompetitive.png, 2Competitive.png 
\begin{figure}[htb]
    \begin{subfigure}{0.33\textwidth}        
        \includegraphics[width=\linewidth]{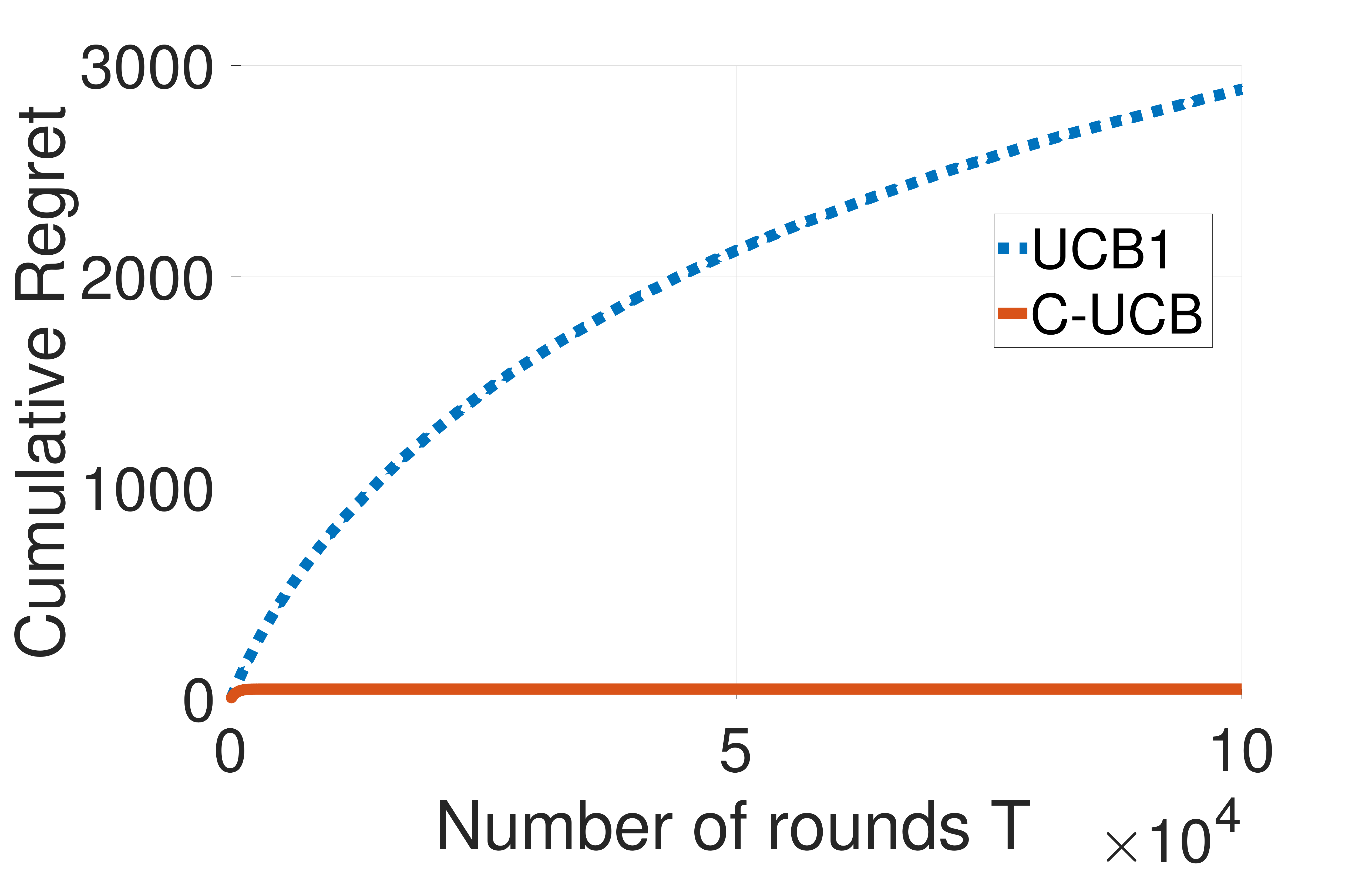}
        \caption{No competitive arms} 
        \label{fig:InvertibleDiscrete}
    \end{subfigure}%
    \begin{subfigure}{0.33\textwidth}
        \includegraphics[width=\linewidth]{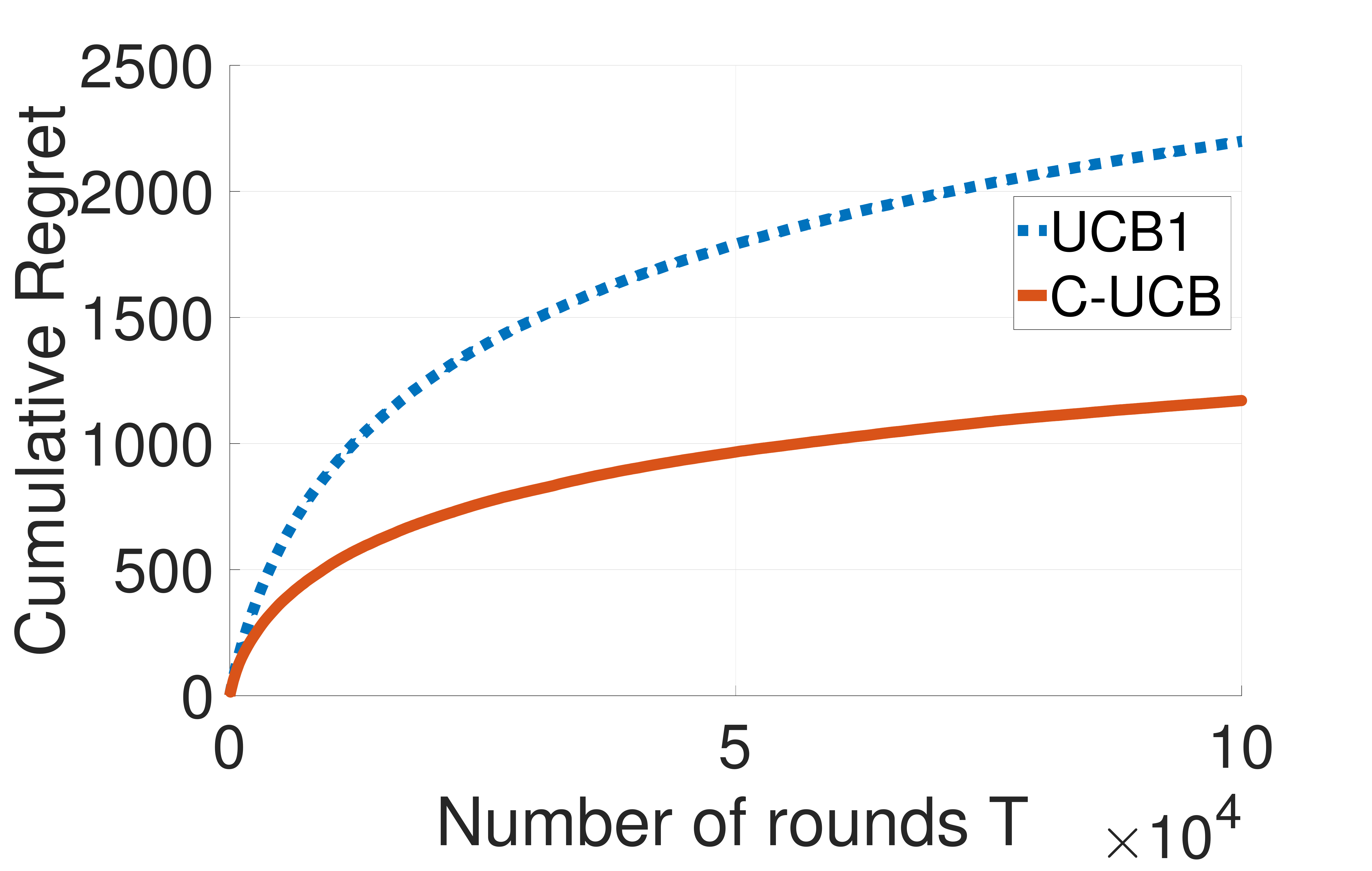}
        \caption{Only One Competitive arm} 
        \label{fig:OneCompetitiveDiscrete}
    \end{subfigure}%
    \begin{subfigure}{0.33\textwidth}
        \includegraphics[width=\linewidth]{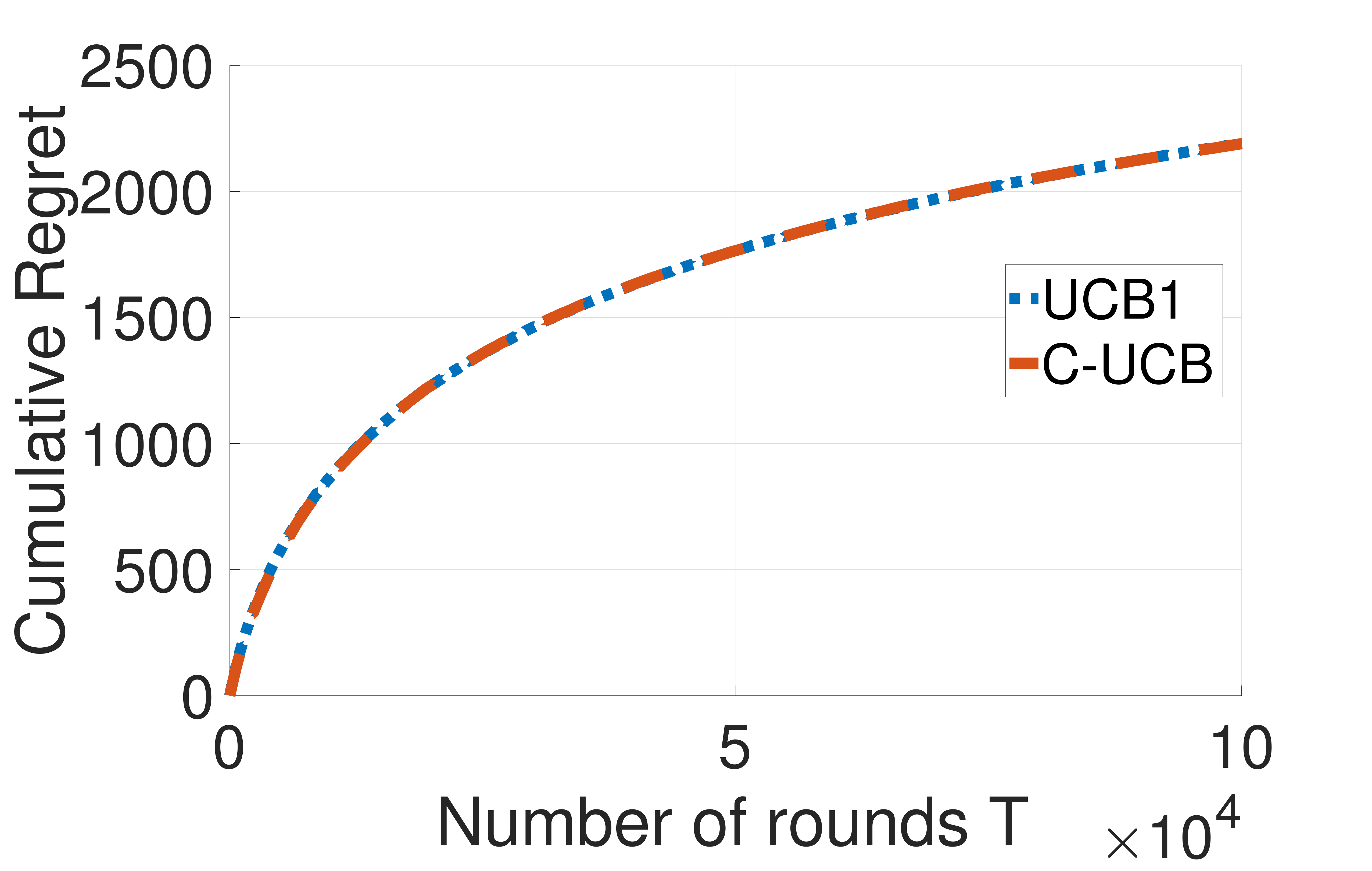}
        \caption{Both Arms are Competitive}
        \label{fig:TwoCompetitiveDiscrete}
    \end{subfigure}
\caption{For the reward functions in \Cref{fig:sim_reward_funcs}, the cumulative regret of C-UCB is smaller than vanilla-UCB1 in all the three cases above.\label{fig:simulationDiscrete}}
\vspace{-0.5cm}
\end{figure}

%\SG{Add the functions and probability distribution for all these three cases}
%\SG{Will add other simulations in the Appendix}

% \begin{itemize}
%     \item Simulation with one invertible arms (optimal) other non invertible
%     \item Simulation with Multiple non invertible arms
%     \item One Simulation with continuous arms
%     \item If possible one simulation with Multi dimensional R.V
%     \item All performances against the UCB algorithm
%     \item Can also include Thompson Sampling in the simulation
% \end{itemize}

%% file: Conclusion.tex
\section{Concluding Remarks}
\label{sec:conclusion}
%\GJ{The abstract and conclusion are too similar. Need to give a broader view of the problem in the conclusions}
%\OY{I am wondering if it is possible for people to get confused and think that there is only one realization of $X$ that generates all rewards, instead of each pull generating an independent realization of $X$.}
%\GJ{I changed the wording to avoid that confusion}
This work studies a correlated multi-armed bandit (MAB) framework where the rewards obtained by pulling the $K$ different arms are functions of a common latent random variable $X$. We propose the C-UCB algorithm which achieves significant regret-reduction over the classic UCB. In fact, C-UCB is able to achieve a constant (instead of the standard logarithmic) regret in certain cases. A key idea behind the success of this algorithm is that correlation helps us use reward samples from one arm to generate pseudo-rewards from other arms, thus obviating the need to explore them. 
%
%samples from arms without pulling arms, hence obviating the need to explore them. from other arms, and t  achieves an expected regret of $C O (\log T) + O(1)$ regret and is order-wise optimal.  a novel generalization of classic UCB, that uses the correlation to identify non-competitive arms in just uses the notion of \emph{competitiveness} of arms to identify \emph{non-competitive} arms, and only perform UCB over a smaller set of competitive arms. C-UCB achieves an expected regret of $C O (\log T) + O(1)$ regret and is order-wise optimal. Thus in certain regimes, C-UCB achieves $O(1)$ expected regret instead of the typical $O(\log T)$ regret scaling.
%
%Ongoing work includes evaluating the performance of the proposed algorithm using a real-world datasets from applications such as ad placement and collaborative filtering.
We believe that this idea is applicable more broadly to several other sequential decision-making problems. Ongoing work includes generalization of other multi-armed bandit algorithms such as Thompson sampling \citep{agrawal2013thompson}, and understanding the scaling of regret with respect to the number of arms $K$. Instead of the deterministic reward functions $g_i(X)$, we also plan to consider random reward variables $Y_i$, such that the conditional distribution $p(Y_i|X)$ is known.

% Main Implications
% \begin{itemize}
%     \item Possible to pull some suboptimal arms for only a finite number of times
%     \item Provide lower bound 
%     \item Design algorithm to do these two jobs and show that we solve the MAB problem in the smallest possible dimension.
% \end{itemize}

% Future directions
% \begin{itemize}
%     \item Try modifying other classical MAB algorithms instead of UCB.
%     \item Study scaling with respect to number of arms. (infinite armed bandits). 
%     \item Instead of deterministic functions $g_i$, having random variables $Y_i$ with given conditional distribution $\Pr(Y_i|X)$
% \end{itemize}

%% file: appendix.tex
%\begin{center}\textbf{SUPPLEMENTARY MATERIAL} \end{center}

\section{Continuous $X$ and Random Vector $\mathbf{X} = (X_1, X_2, \ldots X_m)$}
Observe that our algorithm depends on the functions $g_i(X)$ through the evaluation of pseudo-rewards (see Definition 2). For discrete $X$, the set $\{x: g_{\arm}(x) = \reward\}$ is a discrete set with a finite number of elements. Hence, it is easy to evaluate $\max_{\{x: g_{\arm}(x) = \reward\}}  g_{\ell}(x)$ for any arm $\ell \neq \arm$. For continuous $X$, if $\{x: g_{\arm}(x) = \reward\}$ is a finite union of continuous sets, and if $g_{\ell}(x)$ has finite stationary points, then it is possible to evaluate $g_{\ell}(x)$ for $x$ that lie at the boundary of continuous sets and at stationary points lying within these sets. Therefore, it is possible to compute $\max_{\{x: g_{\arm}(x) = \reward\}}  g_{\ell}(x)$.

% \GJ{The explanation for the continuous case is hand-wavy and needs to be made more rigorous. For example, shouldn't we replace 'continuous sets' by a 'continuous interval of $\mathbb{R}$? And why does $g_{\ell}(x)$ need to have finite stationary points?}
% \SG{How is this version?}

The algorithm and the regret analysis is also applicable to more general random sources, such as a latent random vector $\mathbf{X} = (X_1, X_2, \ldots X_m)$. For example, if $X = (X_1, X_2)$ is a random variable, and $g_1(X) = X_1 + 0.1 X_2$ and $g_2(X) = X_2 + 0.1 X_1$. Then evaluating the pseudo-reward of arm 2 with respect to arm 1 on observing reward $\reward$ reduces to solving an optimization problem
\begin{align*}
\max_{z_1, z_2} &\quad z_2 + 0.1 z_1 \\
\text{s.t }&\quad z_1 + 0.1 z_2 = \reward \\
&\quad z_1 \in \sampleSpace_1, z_2 \in \sampleSpace_2,
\end{align*}
where, $\sampleSpace_1, \sampleSpace_2$ are support of $X_1$ and $X_2$ respectively.

As mentioned in Remark 1, this also captures the case of classical multi-armed bandit problem, if $\mathbf{X} = (X_1, X_2, \ldots X_n)$, where $X_i$ are independent random variables and $g_\arm(X) = X_\arm$ for $\arm \in \{1,2, \ldots K\}$.

\section{Simulations for Continuous X and Random Vector $\mathbf{X}$}
In this section we obtained cumulative regret by averaging over 100 simulation runs, for each run we use the same reward realizations for both the C-UCB and UCB1 (\citep{auer2002finite}) algorithm. We show these results for continuous X and random vector $\mathbf{X}$.
\subsection{Continuous Random Variable}

% \begin{figure}[htb]
%     \begin{subfigure}{0.33\textwidth}
%         \includegraphics[width=\linewidth]{Invertible_Discrete.eps}
%         \caption{No competitive arms} 
%         \label{fig:InvertibleDiscrete}
%     \end{subfigure}%
%     \begin{subfigure}{0.33\textwidth}
%         \includegraphics[width=\linewidth]{One_Competitive_Discrete.eps}
%         \caption{One Competitive and One Non-Competitive arm} 
%         \label{fig:OneCompetitiveDiscrete}
%     \end{subfigure}%
%     \begin{subfigure}{0.33\textwidth}
%         \includegraphics[width=\linewidth]{Two_Competitive_Discrete.eps}
%         \caption{Both Arms are Competitive}
%         \label{fig:TwoCompetitiveDiscrete}
%     \end{subfigure}
% \caption{For the reward functions in \Cref{fig:sim_reward_funcs}, the cumulative regret of C-UCB is smaller than vanilla-UCB1 in all the three cases above.\GJ{Increase font, and make the lines thicker}  \label{fig:simulationDiscrete}}
% \end{figure}

We consider the reward functions $g_1(X), g_2(X)$ and $g_3(X)$ as shown in \Cref{fig:sim_reward_funcs_cont}. Arm 1 corresponds to a Gaussian reward function $g_1(x) = \frac{1}{2 \sqrt{2 \pi \sigma^2}}\exp\left(-\frac{(x - \mu)^2}{2 \sigma^2}\right)$, with $\mu = 0.5$ and $\sigma = 0.2$. Arm 2 corresponds to $g_2(x) = 1 - \exp(-5 \lambda x)$, with $\lambda = 0.5$. Arm 3 corresponds to a uniform reward function with $g_3(x) = 0.5$. Depending on the distribution of random variable $X$, we can have different scenarios. For this simulation, we considered three cases with distribution of $X$ as $Beta(4,4),$ $Beta(2,5)$ and $Beta(1,5)$ respectively. Distribution of $X$ for these three cases is shown in \Cref{fig:sim_distribution_funcs_cont}. 

\begin{figure*}[b!]
\centering
\begin{minipage}[t]{0.4\textwidth}
  \centering
    \includegraphics[width=1.0\textwidth]{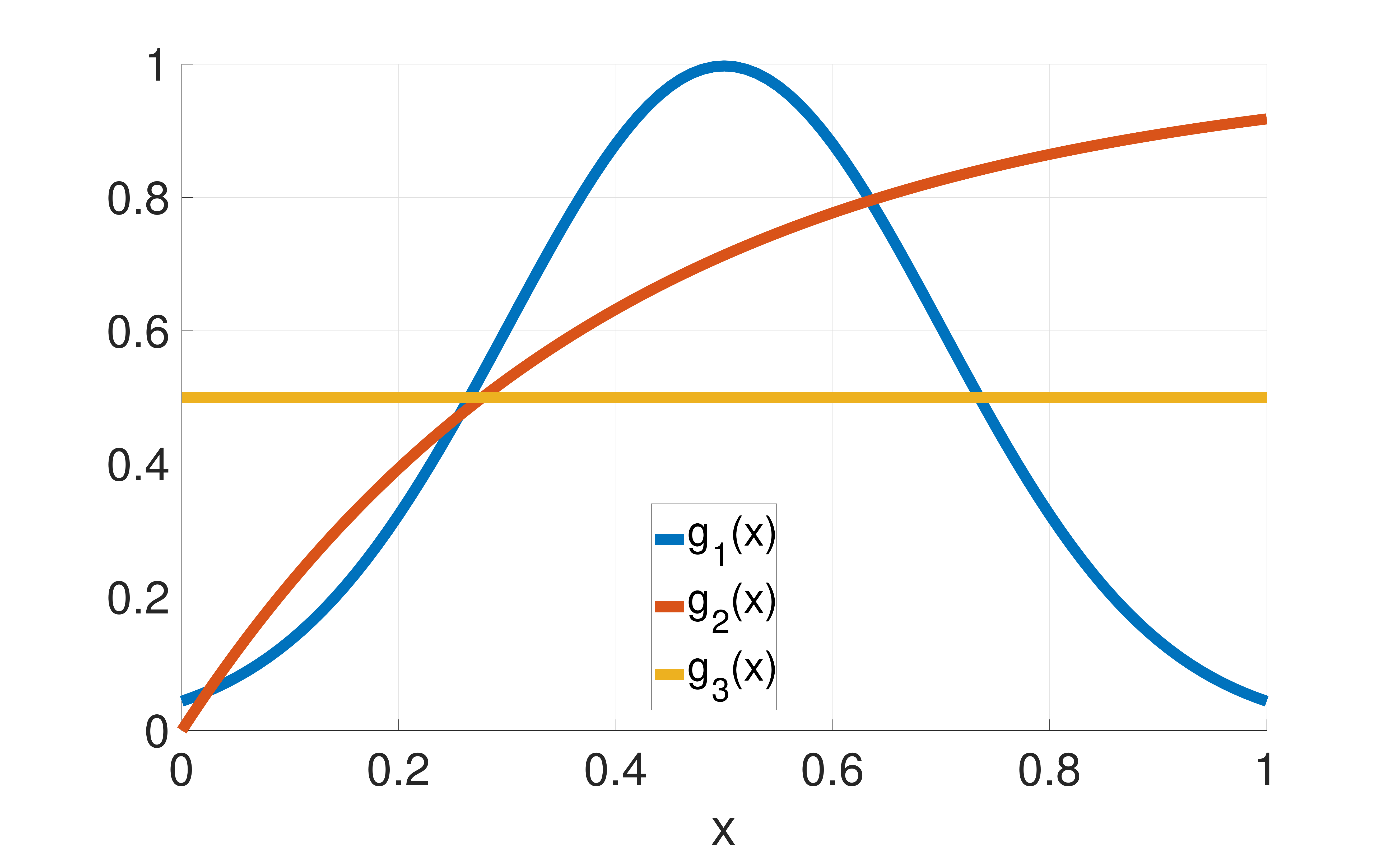}
    \caption{Reward Functions used for the simulation results presented in \Cref{fig:simulationContinuous}.}
    \label{fig:sim_reward_funcs_cont}
\end{minipage}%
~ \hspace{0.5cm}
\begin{minipage}[t]{0.45\textwidth}
  \centering
    \includegraphics[width=0.9\textwidth]{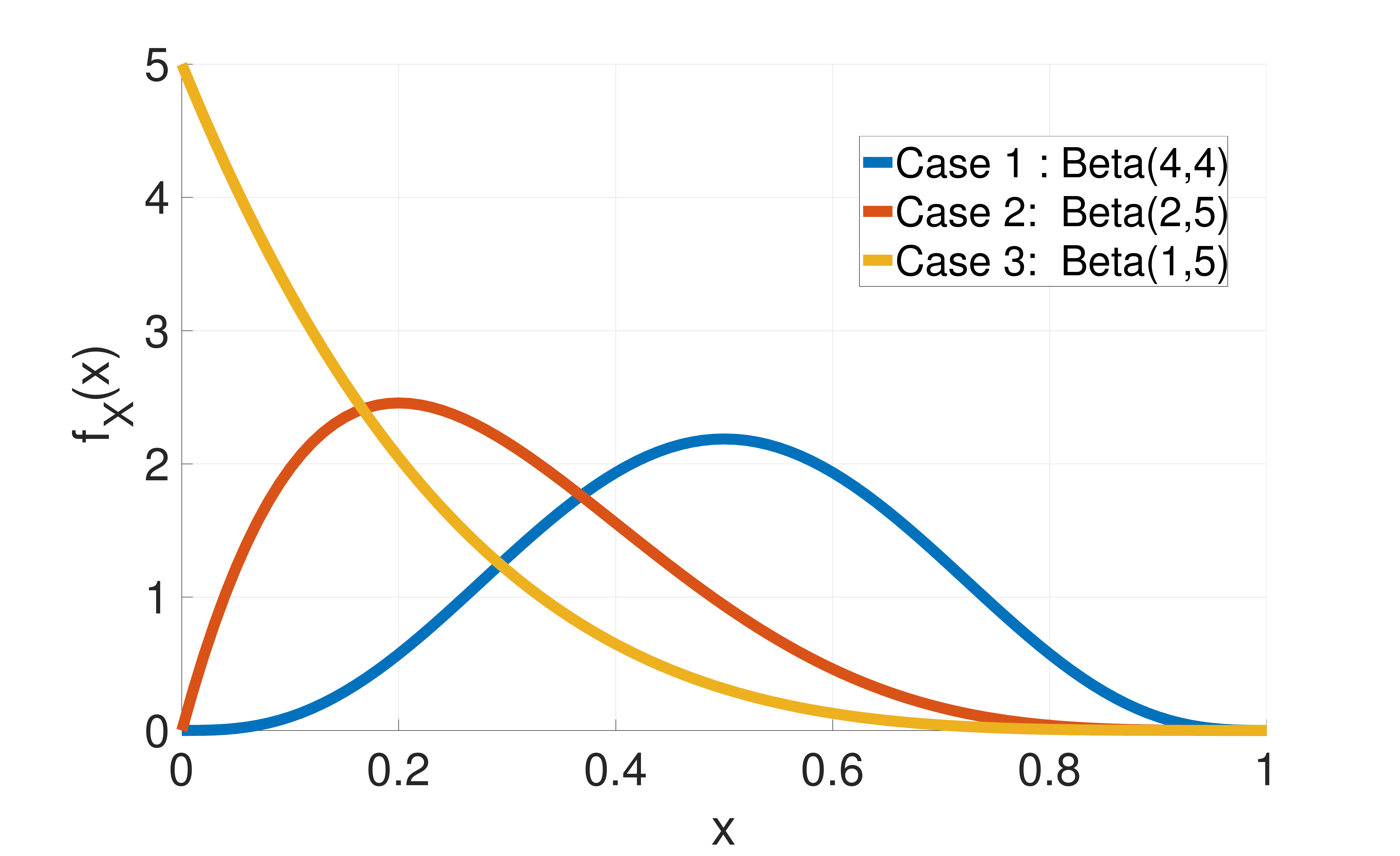}
    \caption{Distribution of $X$ for the three cases of simulation results presented in \Cref{fig:simulationContinuous}.  }
    \label{fig:sim_distribution_funcs_cont}
\end{minipage}
\end{figure*}

\textbf{Case 1: $X \sim Beta(4,4)$.} For this case arm 1 is the optimal arm, and arms 2 and 3 are non-competitive. As a result, the regret of C-UCB algorithm does not scale with the number of rounds. 
% \begin{wrapfigure}{r}{0.4\textwidth}
% \centerline{\includegraphics[width=7.2cm]{functions_continuous.eps}}
% \caption{Reward Functions used for the simulation results presented in ?.} \label{fig:sim_reward_funcs_cont} 
% \vspace{-0.25 cm}
% \end{wrapfigure}
Observe that in \Cref{fig:Continuous1} the regret of C-UCB algorithm is very small; this is because the pseudo-gap of arms 2 and 3 with respect to arm 1 in this setting are large and hence sub-optimal arms are pulled very few times as they are easily identified as sub-optimal through pulls of Arm 1. This also demonstrates a case where sub-optimal arms are \textit{non-nompetitive} even though the optimal arm is non-invertible.

% \begin{wrapfigure}{r}{0.5\textwidth}
% \centerline{\includegraphics[width=7.2cm]{distribution_continuous.eps}}
% \caption{Distribution of $X$ for the three cases discussed.} \label{fig:sim_distribution_funcs_cont} 
% \vspace{-0.25 cm}
% \end{wrapfigure}

\textbf{Case 2: $X \sim Beta(2,5)$.} In this scenario arm 1 is the optimal arm, arm 2 is competitive and arm 3 is \textit{non-competitive}. Due to this, C-UCB algorithm still explores arm 2. As evident in \Cref{fig:Continuous2}, C-UCB clearly outperforms the UCB1 algorithm. This is because C-UCB algorithm explores only arm 2, while UCB1 explores both arm 1 and arm 2.

\textbf{Case 3: $X \sim Beta(1,5)$.} In this case, arm 3 is the optimal arm. Since pulls of arm 3 provide no information about reward from Arm 1 and Arm 2, both Arm 1 and Arm 2 are \textit{Competitive}. Due to this C-UCB algorithm explores both the arms and has a performance very similar to the UCB1 algorithm as shown in \Cref{fig:Continuous3}.

\begin{figure}[htb]
    \begin{subfigure}{0.33\textwidth}
        \includegraphics[width=\linewidth]{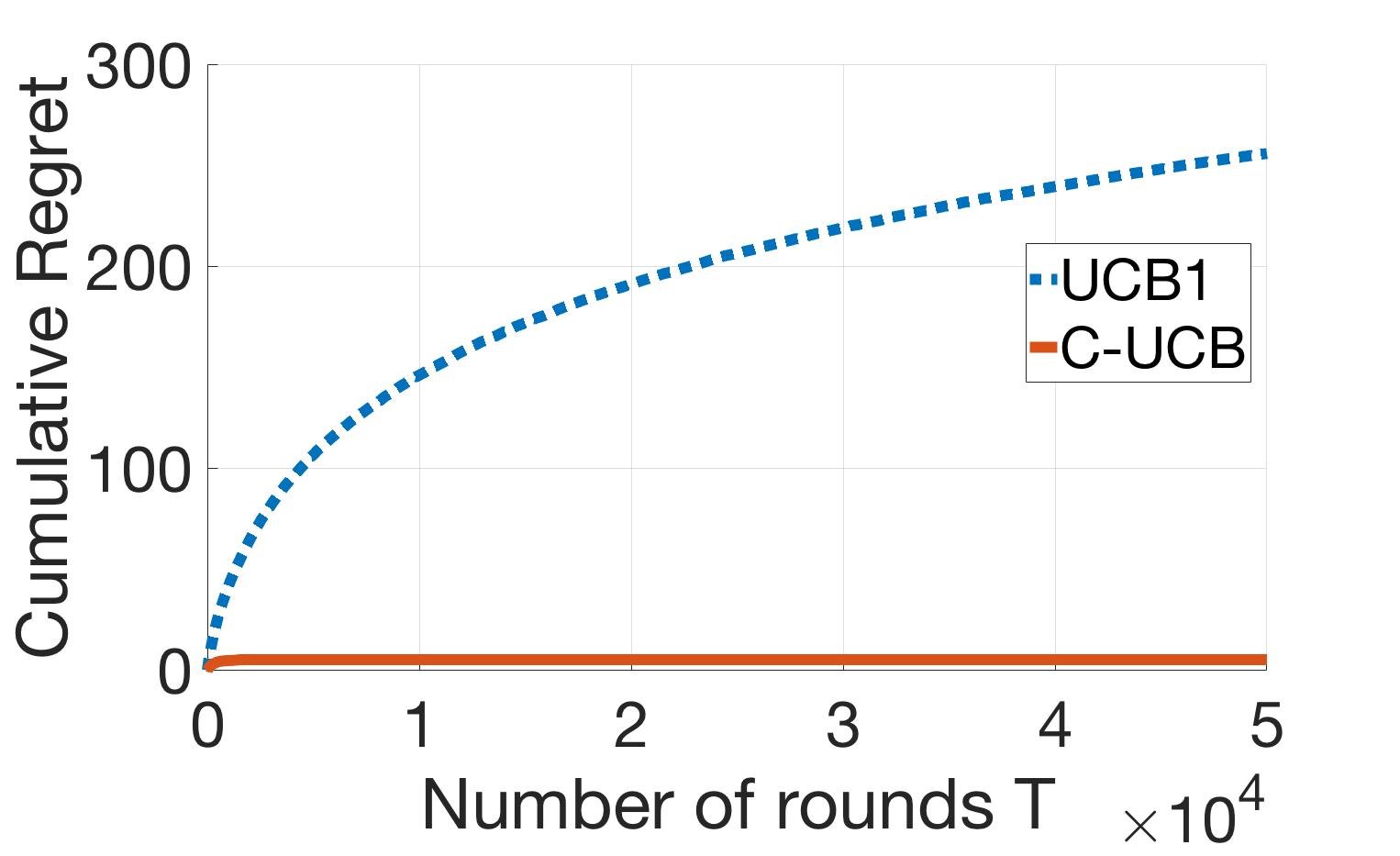}
        \caption{Arm 1 is optimal, other two are non-competitive.} 
        \label{fig:Continuous1}
    \end{subfigure}%
    \begin{subfigure}{0.33\textwidth}
        \includegraphics[width=\linewidth]{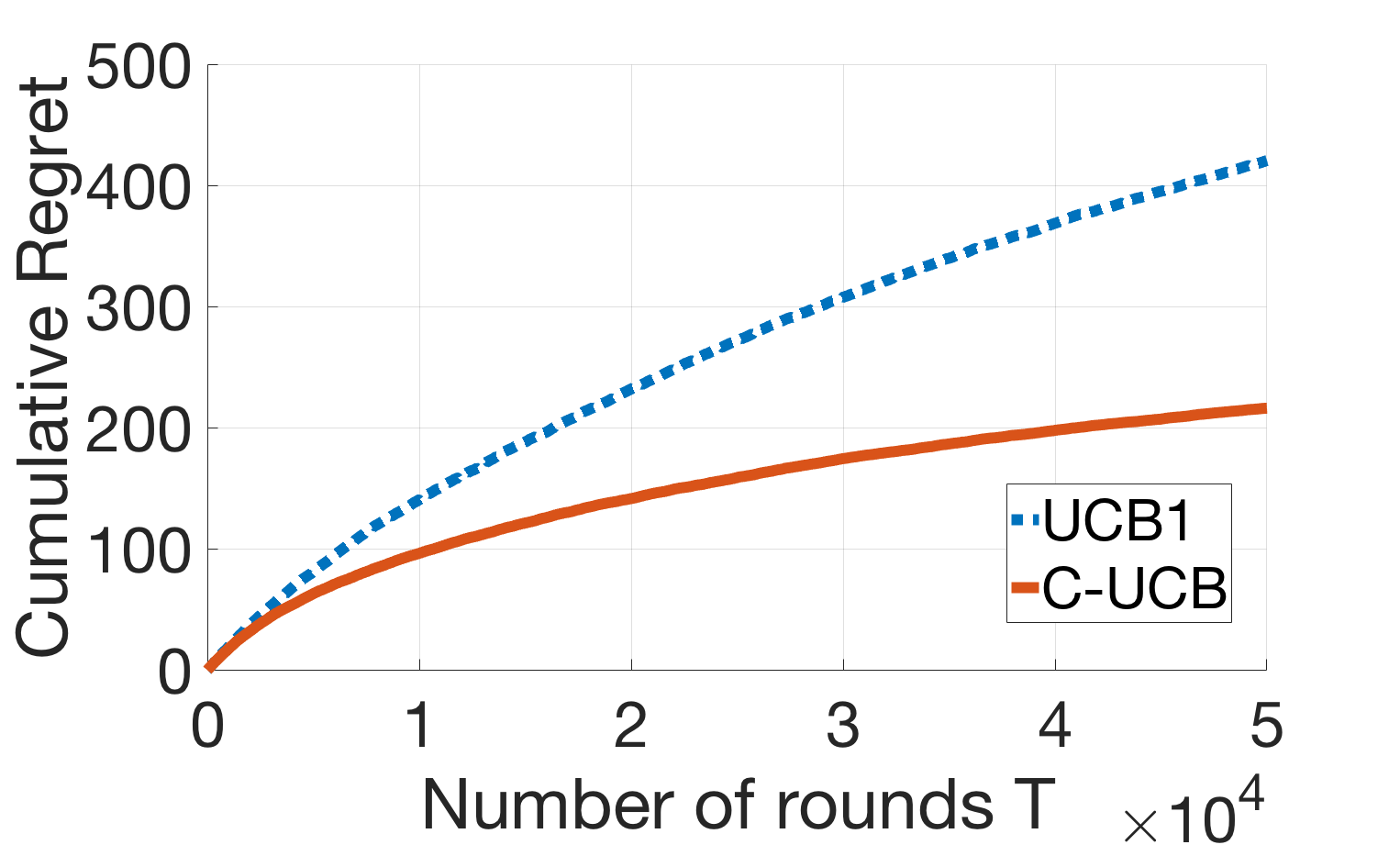}
        \caption{Arm 1 is optimal, arm 2 is competitive and arm 3 is non-competitive.} 
        \label{fig:Continuous2}
    \end{subfigure}%
    \begin{subfigure}{0.33\textwidth}
        \includegraphics[width=\linewidth]{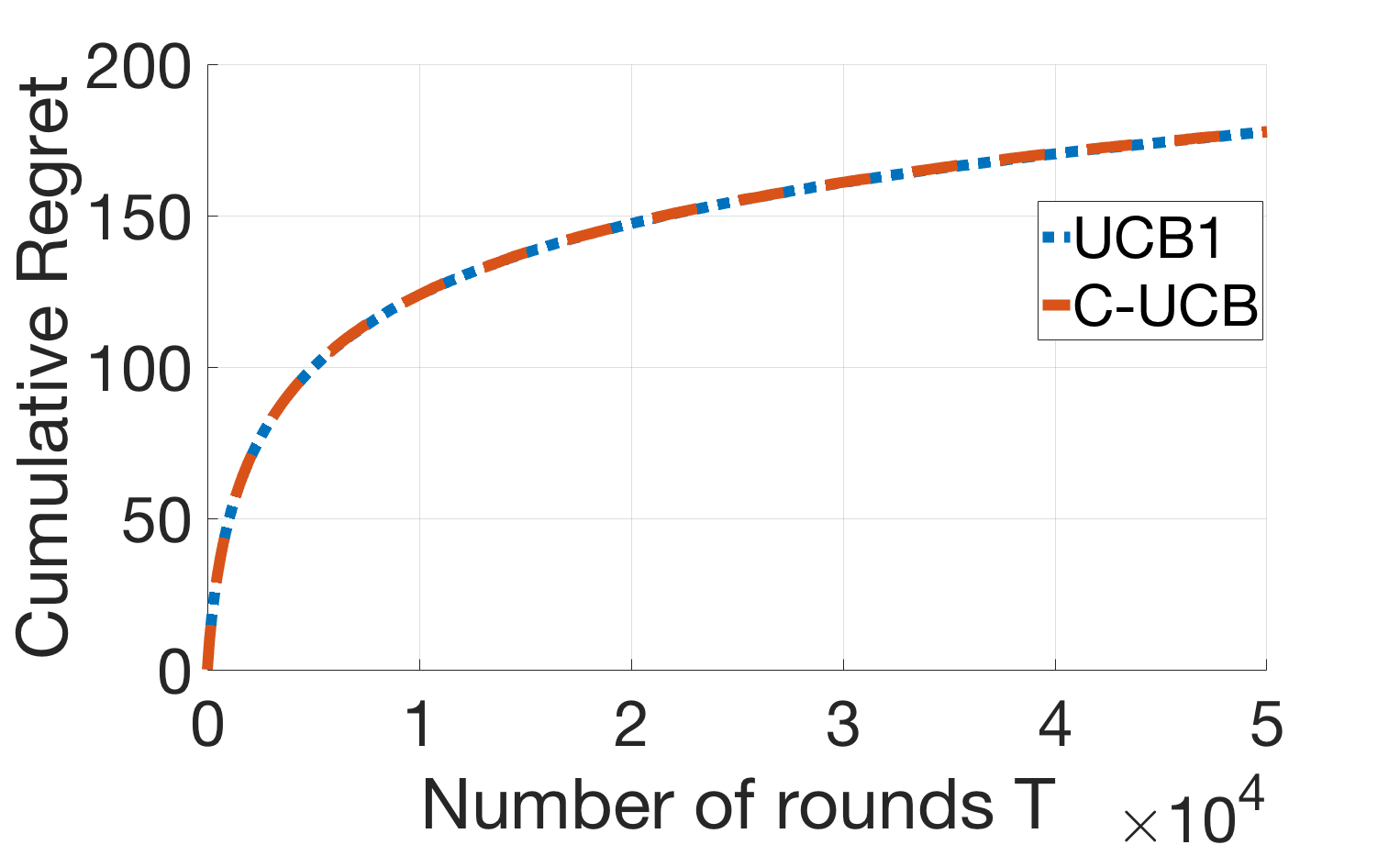}
        \caption{Arm 3 is optimal and arm 1, 2 are competitive.}
        \label{fig:Continuous3}
    \end{subfigure}
\caption{Simulation results for continuous X.} \label{fig:simulationContinuous}
\end{figure}

\subsection{Latent Random Vector $\mathbf{X}$}
We now consider a case where we have a random vector $\mathbf{X} = (X_1,X_2)$. In our setting $X_1, X_2$ have a support of $\{-1,0,1\}$. We consider two arms with $g_1(X) = X_1 + X_2$ and $g_2(X) = X_1 - X_2$. In this example $s_{2,1}(r) > g_1(r)$ only if the observed reward $r = 2$, which corresponds to the case where the realization $(X_1, X_2)$ can be identified as $(1,1)$. Similarly $s_{1,2}(r) > g_2(r)$ only if the observed reward $r = 2$, which corresponds to the realization $(1,-1)$. Depending on the distribution of $X$, suboptimal arm can be competitive or non-competitive.  

\textbf{Case 1: Suboptimal arm is \textit{Competitive}.} We consider a case where $\mathbb{P}_\mathbf{X} = \mathbb{P}_{X_1}\mathbb{P}_{X_2}$, with $\mathbb{P}_{X_1} = \{0.3, 0.4, 0.3\}$ and $\mathbb{P}_{X_2} = \{0.38, 0.22, 0.4\}$. In this scenario Arm 1 is optimal and sub-optimality gap of arm2 is $\gap_2 = 0.04$. Since the probability mass on $(1,1)$ is small, Arm 2 is \textit{Competitive}. Due to this, we see in \Cref{fig:Multi2} that regret of the C-UCB algorithm scales with number of rounds $T$ and has a performance very similar to the UCB1 algorithm. 

\textbf{Case 2: Suboptimal arm is \textit{Non-Competitve}} We consider the distribution $\mathbb{P}_{\mathbf{X}}$ with $\mathbb{P}_{\mathbf{X}}(1,-1) = 0.48$, $\mathbb{P}_{\mathbf{X}}(1,1) = 0.5$ and $\mathbb{P}_{\mathbf{X}}(x_1,x_2) = 0.0028$ for all other $x_1,x_2$. In this scenario, arm $1$ is optimal and arm $2$ is sub-optimal with suboptimality gap $\gap_2 = 0.04$. Since probability mass at $(1,1)$ is high, it is possible to infer sub-optimality of arm 2 using reward samples of arm 1. We see this effect in \Cref{fig:Multi1}.

\begin{figure}[htb]
    \begin{subfigure}{0.5\textwidth}
        \includegraphics[width=\linewidth]{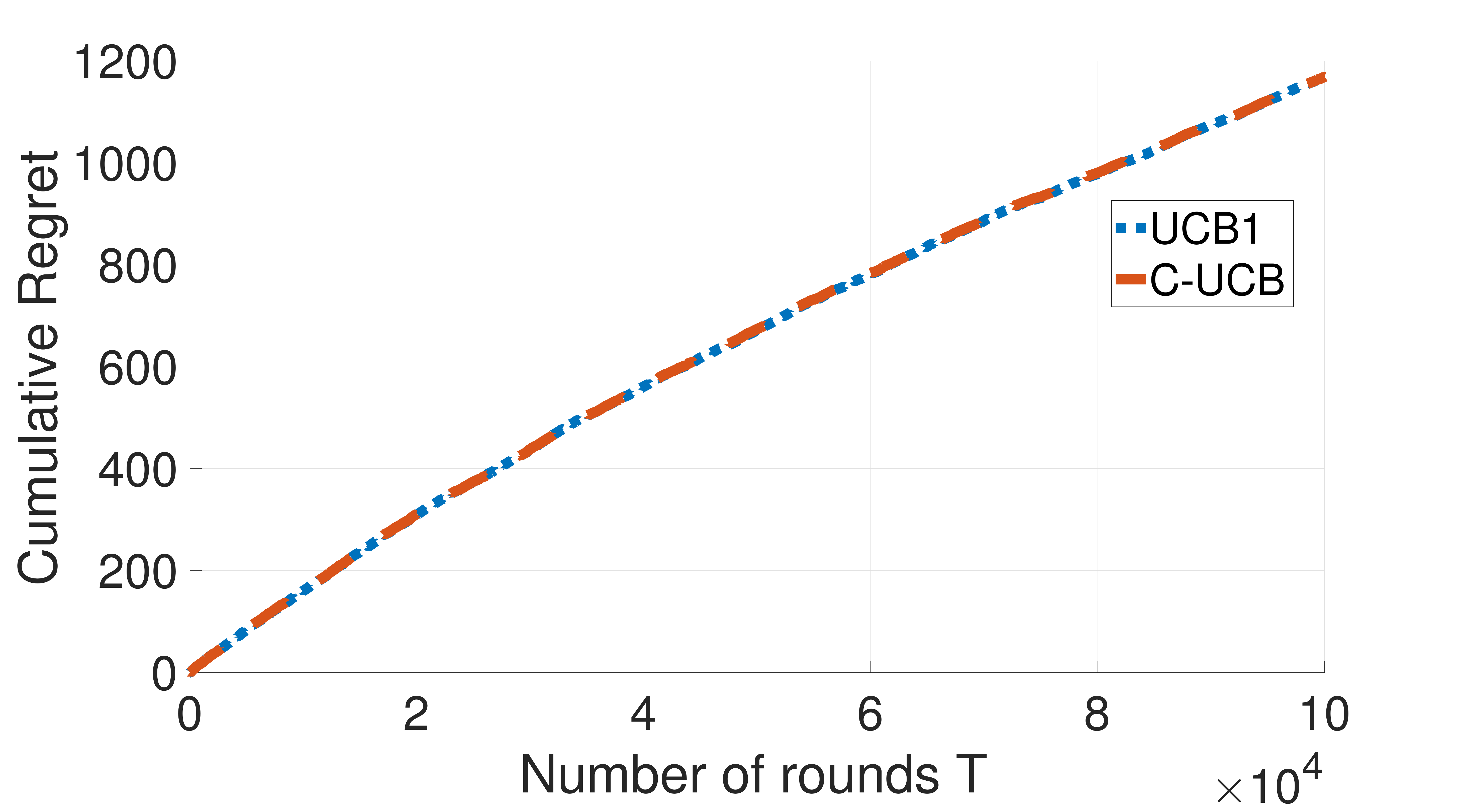}
        \caption{Sub-optimal arm is competitive.} 
        \label{fig:Multi2}
    \end{subfigure}%
    \begin{subfigure}{0.5\textwidth}
        \includegraphics[width=\linewidth]{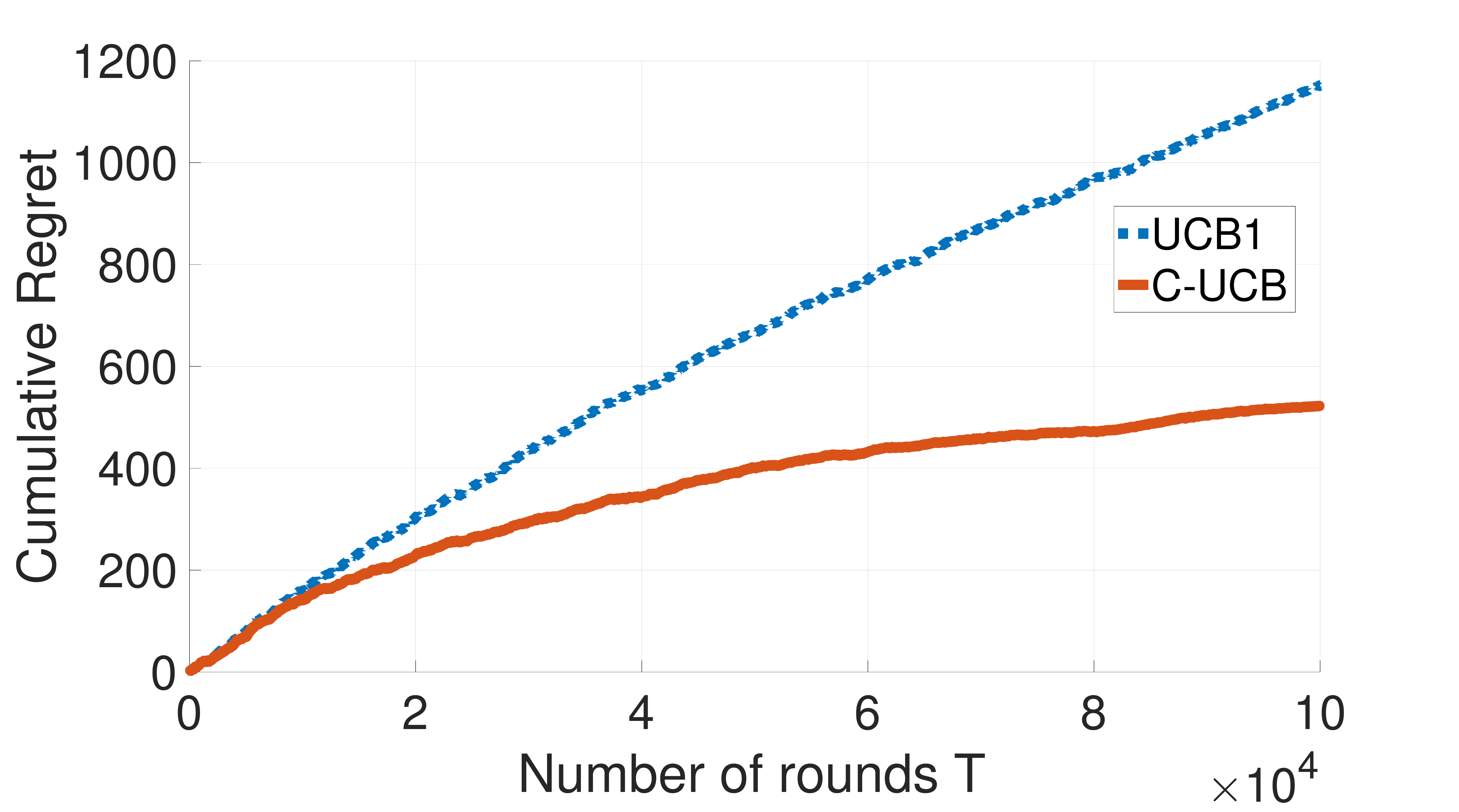}
        \caption{Sub-optimal arm is non-competitive.} 
        \label{fig:Multi1}
    \end{subfigure}%

\caption{Simulation results for latent vector $\mathbf{X}$.} \label{fig:simulationMulti}
\end{figure}

\section{Standard Results from Previous Works}
\begin{fact}[Hoeffding's inequality]
Let $Z_1, Z_2 \ldots Z_n$ be i.i.d random variables bounded between $[a, b]: a \leq Z_i \leq b$, then for any $\delta>0$, we have
$$\Pr\left(\left| \frac{\sum_{i = 1}^{n} Z_i}{n} - \E{Z_i} \right| \geq \delta\right) \leq \exp \left( \frac{-2 n \delta^2}{(b - a)^2}\right).$$ 
\end{fact}

\begin{lem}[Standard result used in bandit literature]
If $\hat{\mu}_{k,n_k(t)}$ denotes the empirical mean of arm $k$ by pulling arm $k$ $n_k(t)$ times through any algorithm and $\mu_k$ denotes the mean reward of arm $k$, then we have 
$$\Pr\left(\hat{\mu}_{k,n_k(t)} - \mu_k \geq \epsilon, \tau_2 \geq n_k(t) \geq \tau_1 \right) \leq \sum_{s = \tau_1}^{\tau_2}\exp \left(- 2 s \epsilon^2\right).$$ 
\label{lem:UnionBoundTrickInt}
\end{lem}

\begin{proof}
Let $Z_1, Z_2, ... Z_t$ be the reward samples of arm $k$ drawn separately. If the algorithm chooses to play arm $k$ for $m^{th}$ time, then it observes reward $Z_m$. Then the probability of observing the event $\hat{\mu}_{k,n_k(t)} - \mu_k \geq \epsilon, \tau_2 \geq n_k(t) \geq \tau_1$ can be upper bounded as follows,
\begin{align}
    \Pr\left(\hat{\mu}_{k,n_k(t)} - \mu_k \geq \epsilon, \tau_2 \geq n_k(t) \geq \tau_1 \right) &= \Pr\left( \left( \frac{\sum_{i=1}^{n_k(t)}Z_i}{n_k(t)} - \mu_k \geq \epsilon \right), \tau_2 \geq n_k(t) \geq \tau_1 \right) \\
    &\leq \Pr\left( \left(\bigcup_{m = \tau_1}^{\tau_2} \frac{\sum_{i=1}^{m}Z_i}{m} - \mu_k \geq \epsilon \right), \tau_2 \geq n_k(t) \geq \tau_1 \right) \label{upperBoundTrick}\\
    &\leq \Pr \left(\bigcup_{m = \tau_1}^{\tau_2} \frac{\sum_{i=1}^{m}Z_i}{m} - \mu_k \geq \epsilon \right) \\
    &\leq \sum_{s = \tau_1}^{\tau_2}\exp \left( - 2 s \epsilon^2\right).
\end{align}
\end{proof}

\begin{lem}[From Proof of Theorem 1 in \citep{auer2002finite}]
\label{lem:ucbindexmore}
Let $\Index_\arm(\slot)$ denote the UCB index of arm $\arm$ at round $\slot$, and $\meanReward_\arm = \E{g_{\arm}(X)}$ denote the mean reward of that arm.  Then, we have
$$\Pr(\meanReward_\arm > \Index_\arm(\slot)) \leq \slot^{-3}.$$ 
\end{lem}
Observe that this bound does not depend on the number $\pulls_\arm(\slot)$ of times arm $\arm$ is pulled. UCB index is defined in equation (6) of the main paper.
\begin{proof}
This proof follows directly from \citep{auer2002finite}. We present the proof here for completeness as we use this frequently in the paper.
\begin{align}
    \Pr(\meanReward_\arm > \Index_\arm(\slot)) &= \Pr\left(\meanReward_\arm > \hat{\meanReward}_{\arm,\pulls_\arm(\slot)} + \sqrt{\frac{2 \log \slot}{\pulls_\arm(\slot)}}\right) \\
    &\leq \sum_{m = 1}^{\slot} \Pr \left(\meanReward_\arm > \hat{\meanReward}_{\arm,m} + \sqrt{\frac{2 \log \slot}{m}} \right) \label{unionTrick}\\
    &= \sum_{m =1}^{\slot} \Pr \left(\hat{\meanReward}_{\arm,m} - \meanReward_\arm < - \sqrt{\frac{2 \log \slot}{m}}\right) \\ 
    &\leq \sum_{m = 1}^{\slot} \exp\left(- 2 m \frac{2 \log \slot}{m}\right) \label{eqn:ucbindex}\\
    &= \sum_{m = 1}^{\slot} \slot^{-4} \\
    &= \slot^{-3}.
\end{align}
where \eqref{unionTrick} follows from the union bound and is a standard trick (\Cref{lem:UnionBoundTrickInt}) to deal with random variable $\pulls_\arm(\slot)$. We use this trick repeatedly in the proofs. We have \eqref{eqn:ucbindex} from the Hoeffding's inequality. 
\end{proof}

\begin{lem} Let $\E{\indicator_{\Index_\arm > \Index_{\arm^*}}}$  be the expected number of times $\Index_\arm (t)> \Index_{\arm^*}(t)$ in $\totalPulls$ rounds. Then, we have 
$$\E{\indicator_{\Index_\arm > \Index_{\arm^*}}} = \sum_{\slot = 1}^{\totalPulls} Pr(\Index_\arm > \Index_{\arm^*}) \leq \frac{8 \log (\totalPulls)}{\gap_\arm^2} + \left(1 + \frac{\pi^2}{3} \right).$$
\label{lem:AuerResult}
\end{lem}

The proof follows the analysis in Theorem 1 of \citep{auer2002finite}. The analysis of  $Pr(\Index_\arm > \Index_{\arm^*})$ is done by conditioning on the event that Arm $\arm$ has been pulled $\frac{8 \log (\totalPulls)}{\gap_\arm^2}$. Conditioned on this event, $Pr(\Index_\arm(\slot) > \Index_{\arm^*}(\slot) | \pulls_\arm(\slot)) \leq \slot^{-2}$.

\begin{lem}[Theorem 2 \citep{lai1985asymptotically}]

Consider a two armed bandit problem with reward distributions $\Theta = \{f_{R_1}(\reward), f_{R_2}(\reward)\}$, where the reward distribution of the optimal arm is $f_{R_1}(\reward)$ and for the sub-optimal arm is $f_{R_2}(\reward)$, and $\E{f_{R_1}(\reward)} > \E{f_{R_2}(\reward)}$; i.e., arm 1 is optimal. If it is possible to create an alternate problem with distributions $\Theta' = \{f_{R_1}(\reward), \tilde{f}_{R_2}(\reward)\}$ such that $\E{\tilde{f}_{R_2}(\reward)} > \E{f_{R_1}(\reward)}$ and $0< D(f_{R_2}(r)||\tilde{f}_{R_2}(r)) < \infty$ (equivalent to assumption 1.6 in \citep{lai1985asymptotically}),  then for any policy that achieves sub-polynomial regret, we have $$\liminf\limits_{\totalPulls \rightarrow \infty}  \frac{\E{\pulls_2(\totalPulls)}}{\log \totalPulls} \geq \frac{1}{D(f_{R_2}(r) || \tilde{f}_{R_2}(r))}.$$

% Assumption 1.7 in this setting refers to: $\forall \epsilon > 0$, such that $\E{\tilde{\rewardDist}_2(\reward)} > \E{\rewardDist_1(\reward)}$, $\exists \delta = \delta(\epsilon, \rewardDist_1(r), \tilde{\rewardDist}_2(r)) > 0$ for which $|D(\rewardDist_2(r) || \rewardDist_1(r)) -   D(\rewardDist_2(r) || \tilde{\rewardDist}_2(r))| \leq \epsilon$ whenever $\E{\rewardDist_1(\reward)} \leq \E{\tilde{\rewardDist}_2(\reward)} \leq \E{\rewardDist_1(\reward)} + \delta$
\label{lem:LaiRobbins2Arms}
\end{lem}

\begin{proof}
Proof of this is derived from the analysis done in \citep{banditalgs}. We show the analysis here for completeness. A bandit instance $v$ is defined by the reward distribution of arm 1 and arm 2. Since policy $\pi$ achieves sub-polynomial regret, for any instance $v$, $\mathbb{E}_{v,\pi}\left[(\regret(\totalPulls))\right] = \OO(\totalPulls^p)$ as $\totalPulls \rightarrow \infty$, for all $p > 0$. 

Consider the bandit instances $\Theta = \{f_{R_1}(r), f_{R_2}(r)\}$, $\Theta' = \{f_{R_1}(r), \tilde{f}_{R_2}(r)\}$, where $\E{f_{R_2}(r)} < \E{f_{R_1}(r)} < \E{\tilde{f}_{R_2}(r)}$. The bandit instance $\Theta'$ is constructed by changing the reward distribution of arm 2 in the original instance, in such a way that arm 2 becomes optimal in instance $\Theta'$ without changing the reward distribution of arm 1 from the original instance.

From divergence decomposition lemma (derived in \citep{banditalgs}), it follows that $$D(\mathbb{P}_{\Theta,\Pi} || \mathbb{P}_{\Theta',\Pi}) = \mathbb{E}_{\Theta,\pi}\left[\pulls_2(\totalPulls)\right] D(f_{R_2}(r) || \tilde{f}_{R_2}(r)).$$

The high probability Pinsker's inequality (Lemma 2.6 from \citep{Tsybakov:2008:INE:1522486}, originally in \citep{highProbPinsker}) gives that for any event $A$, $$\mathbb{P}_{\Theta, \pi}(A) + \mathbb{P}_{\Theta',\pi}(A^c) \geq \frac{1}{2}\exp\left(-D(\mathbb{P}_{\Theta,\pi} || \mathbb{P}_{\Theta',\pi})\right),$$
or equivalently, 
$$D(\mathbb{P}_{\Theta,\pi} || \mathbb{P}_{\Theta',\pi}) \geq \log \frac{1}{2(\mathbb{P}_{\Theta,\pi}(A) + \mathbb{P}_{\Theta',\pi}(A^c))}.$$

If arm 2 is suboptimal in a 2-armed bandit problem, then $\E{\regret(\totalPulls)} = \gap_2 \E{\pulls_2(\totalPulls)}.$ Expected regret in $\Theta$ is $$\mathbb{E}_{\Theta,\pi}\left[\regret(\totalPulls)\right] \geq \frac{\totalPulls \gap_2}{2} \mathbb{P}_{\Theta,\pi}\left(\pulls_2(\totalPulls) \geq \frac{\totalPulls}{2}\right),$$
Similarly regret in bandit instance $\Theta'$ is $$\mathbb{E}_{\Theta',\pi}\left[\regret(\totalPulls)\right] \geq \frac{\totalPulls \delta}{2} \mathbb{P}_{\Theta',\pi}\left(\pulls_2(\totalPulls) < \frac{\totalPulls}{2}\right),$$
since suboptimality gap of arm $1$ in $\Theta'$ is  $\delta$. Define $\kappa(\gap_2, \delta) = \frac{\min(\gap_2, \delta)}{2}$. Then we have, $$\mathbb{P}_{\Theta,\pi}\left(\pulls_2(\totalPulls) \geq \frac{\totalPulls}{2}\right) + \mathbb{P}_{\Theta',\pi}\left(\pulls_2(\totalPulls) < \frac{\totalPulls}{2}\right) \leq \frac{\mathbb{E}_{\Theta,\pi}\left[\regret(\totalPulls)\right] + \mathbb{E}_{\Theta',\pi}\left[\regret(\totalPulls)\right]}{\kappa(\gap_2, \delta) \totalPulls}.$$

On applying the high probability Pinsker's inequality and divergence decomposition lemma stated earlier, we get 
\begin{align}
    D(f_{R_2}(r) || \tilde{f}_{R_2}(r)) \mathbb{E}_{\Theta,\pi}\left[\pulls_2(\totalPulls)\right] &\geq \log\left(\frac{\kappa(\gap_2, \delta) \totalPulls}{2 (\mathbb{E}_{\Theta,\pi}\left[\regret(\totalPulls)\right] + \mathbb{E}_{\Theta',\pi}\left[\regret(\totalPulls)\right]) }\right) \\
    &= \log\left(\frac{\kappa(\gap_2,\delta)}{2}\right) + \log(\totalPulls)  \nonumber \\ 
    &\qquad   - \log(\mathbb{E}_{\Theta,\pi}\left[\regret(\totalPulls)\right] + \mathbb{E}_{\Theta',\pi}\left[\regret(\totalPulls)\right]).
\end{align}

Since policy $\pi$ achieves sub-polynomial regret for any bandit instance, $\mathbb{E}_{\Theta,\pi}\left[\regret(\totalPulls)\right] + \mathbb{E}_{\Theta',\pi}\left[\regret(\totalPulls)\right] \leq \gamma \totalPulls^p$ for all $\totalPulls$ and any $p > 0$, 
% \OY{I think sub-polynomial should imply something stronger than this; e.g., Regret $=o(T^p)$ for any $p>0$.}
hence, 
\begin{align} 
\liminf\limits_{\totalPulls \rightarrow \infty}  D(f_{R_2}(r) || \tilde{f}_{R_2}(r)) \frac{\mathbb{E}_{\Theta,\pi}\left[\pulls_2(\totalPulls)\right]}{\log \totalPulls} &\geq 1 - \limsup\limits_{\totalPulls \rightarrow \infty}  \frac{\mathbb{E}_{\Theta,\pi}\left[\regret(\totalPulls)\right] + \mathbb{E}_{\Theta',\pi}\left[\regret(\totalPulls)\right]}{\log \totalPulls} + \nonumber \\
&\quad \liminf\limits_{\totalPulls \rightarrow \infty}  \frac{\log\left(\frac{\kappa(\Delta_2,\delta)}{2}\right)}{\log \totalPulls} \\
&= 1.
\end{align}

Hence, $\liminf\limits_{\totalPulls \rightarrow \infty} \frac{\mathbb{E}_{\Theta,\pi}\left[\pulls_2(\totalPulls)\right]}{\log \totalPulls} \geq \frac{1}{D(f_{R_2}(r) || \tilde{f}_{R_2}(r))}.$

\end{proof}

\section{Lemmas Required to Prove Theorems 1, 2, 3, and 5}
\begin{lem}
Define $E_1(\slot)$ to be the event that arm $\arm^*$ is empirically \textit{non-competitive} in round $\slot+1$, then, 
$$\Pr(E_1(\slot)) \leq \slot \exp \left(\frac{-\slot \gap_{\text{min}}^2}{2 \numArms}\right),$$
where $\gap_{\text{min}} = \min_\arm \gap_\arm$, the gap between the best and second-best arms.
\label{lem:eliminatedOptimal}
\end{lem}

\begin{proof}
We analyze the probability that arm $\arm^*$ is empirically non competitive by conditioning on the event that arm $\arm^*$ is not pulled for maximum number of times till round $\slot$. Analyzing this expression gives us, 
\begin{align}
    \Pr(E_1(t)) &=  \Pr(E_1(\slot) , \pulls_{\arm^*}(\slot) \neq \max_\arm \pulls_\arm(\slot)) \\
    &= \sum_{\arm \neq \arm^*} Pr(E_1(\slot) , \pulls_\arm(t) = \max_{\arm'} \pulls_{\arm'}(\slot)) \\
    &\leq \max_\arm \Pr(E_1(\slot) , \pulls_\arm(\slot) = \max_{\arm'} \pulls_{\arm'}(t)) \\
    &= \max_\arm \Pr(\hat{\meanReward}_\arm > \estimateMean_{\arm^*,\arm} , \pulls_\arm(\slot) = \max_{\arm'} \pulls_{\arm'}(t)) \label{eliminatedOptimal} \\
    &\leq \max_\arm \Pr\left(\hat{\meanReward}_\arm > \estimateMean_{\arm^*,\arm} ,  \pulls_\arm(\slot) \geq \frac{\slot}{\numArms}\right) \label{itIsmax}\\
    &= \max_\arm \Pr \left(\frac{\sum_{\tau = 1}^{\slot} \indicator_{\{\arm_\tau = \arm\}} \reward_\tau}{\pulls_\arm(\slot)} > \frac{\sum_{\tau = 1}^{\slot} \indicator_{\{\arm_\tau = \arm\}} \estimateReward_{\arm^*,\arm}(\reward_\tau)}{\pulls_\arm(\slot)} , \pulls_\arm(\slot) \geq \frac{\slot}{\numArms}\right) 
\end{align}
\begin{align} 
    &= \max_\arm \Pr\left(\frac{\sum_{\tau = 1}^{\slot}\indicator_{\{\arm_\tau = \arm\}}\left(\reward_\tau - \estimateReward_{\arm^*,\arm}(\reward_\tau) \right)}{\pulls_\arm(\slot)} > 0 , \pulls_\arm(\slot) \geq \frac{\slot}{\numArms}\right)  \\
    &= \max_\arm \Pr\left(\frac{\sum_{\tau = 1}^{\slot}\indicator_{\{\arm_\tau = \arm\}}\left(\reward_\tau - \estimateReward_{\arm^*,\arm}(\reward_\tau) \right)}{\pulls_\arm(\slot)} - (\meanReward_\arm - \expectedPseudoReward_{\arm^*, \arm}) > \expectedPseudoReward_{\arm^*, \arm} - \meanReward_\arm , \pulls_\arm(\slot) \geq \frac{\slot}{\numArms}\right)  \\
    &\leq \max_\arm \Pr\left(\frac{\sum_{\tau = 1}^{\slot}\indicator_{\{\arm_\tau = \arm\}}\left(\reward_\tau - \estimateReward_{\arm^*,\arm}(\reward_\tau) \right)}{\pulls_\arm(\slot)} - (\meanReward_\arm - \expectedPseudoReward_{\arm^*, \arm}) > \gap_\arm , \pulls_\arm(\slot) \geq \frac{\slot}{\numArms}\right) \label{empMoreThanPseudo} \\    
    &\leq \max_\arm \slot \exp \left(\frac{- \slot \gap_\arm^2}{2 \numArms}\right) \label{eqn:chernoff} \\
    &= \slot \exp\left(\frac{- \slot \gap_{\text{min}}^2}{2 \numArms}\right),
   \end{align}
%\GJ{Can express (16) as a sum over all arms $k \neq k^*$}
%\eqref{dropIneq} follows since $Pr(\pulls_{\arm^*}(\slot) \neq \max_\arm \pulls_\arm(\slot)) \leq 1$.\\
Here \eqref{eliminatedOptimal} follows from the fact that in order for arm $\arm^*$ to be empirically non-competitive, empirical mean of arm $\arm$ should be more than empirical pseudo-reward of arm $\arm^*$ with respect to arm $\arm$. Inequality \eqref{itIsmax} follows since $\pulls_\arm(\slot)$ being more than $\frac{\slot}{\numArms}$ is a necessary condition for $\pulls_\arm(\slot) = \max_{\arm'} \pulls_{\arm'}(\slot)$ to occur. We have \eqref{empMoreThanPseudo} as $\estimateReward_{\arm^*,\arm}$ is more than $\meanReward_{\arm^*}$. We have \eqref{eqn:chernoff} from the Hoeffding's inequality, as we note that rewards $\{\reward_\tau - \estimateReward_{\arm^*,\arm}(\reward_\tau): \tau=1,\ldots, t, ~ k_{\tau}=k\}$ form a collection of i.i.d. random variables each of which is bounded between $[-1,1]$ with mean $(\meanReward_\arm - \expectedPseudoReward_{\arm^*, \arm})$. The term $\slot$ before the exponent in \eqref{eqn:chernoff} arises as the random variable $\pulls_\arm(\slot)$ can take values from $\slot/\numArms$ to $\slot$ (\Cref{lem:UnionBoundTrickInt}).
%\GJ{Add labels to the equations above, and refer to the respective equation numbers when describing how they are obtained.}
\end{proof}
\begin{lem}
If $\gap_{\text{min}} \geq 4\sqrt{\frac{\numArms \log \slot_0}{\slot_0}}$ for some constant $\slot_0 > 0$, then,
$$\Pr(\arm_{\slot+1} = \arm , \pulls_\arm(\slot) \geq s) \leq 3 \slot^{-3} \quad \text{for } s > \frac{\slot}{2 \numArms}, \forall \slot > \slot_0.$$
\label{lem:noMorePulls}
\end{lem}

\begin{proof}
By noting that $\arm_{\slot + 1} = \arm$ corresponds to arm $\arm$ having the highest index among the set of arms that are not empirically \textit{non-competitive} (denoted by $\mathcal{A}$), we have,  
\begin{align}
    \Pr(\arm_{\slot + 1} = \arm , \pulls_\arm(\slot) \geq s) &= \Pr(\Index_\arm(\slot) = \arg \max_{\arm' \in \mathcal{A}} \Index_{\arm'}(\slot) , \pulls_\arm(\slot) \geq s) \\
    &\leq \Pr(E_1(\slot) \cup \left(E_1^c(\slot), \Index_\arm(\slot) > \Index_{\arm^*}(\slot)\right) , \pulls_\arm(\slot) \geq s) \label{eliminating1}\\ 
    &\leq \Pr(E_1(\slot) , \pulls_\arm(\slot) \geq s) + \Pr(E_1^c(\slot), \Index_\arm(\slot) > \Index_{\arm^*}(\slot) , \pulls_\arm(\slot) \geq s ) \label{unionBound}\\
    &\leq \slot \exp\left(\frac{-\slot \gap_{\text{min}}^2}{2 \numArms}\right) + \Pr\left(\Index_\arm(\slot) > \Index_{\arm^*}(\slot) , \pulls_\arm(\slot) \geq s\right). \label{usedHoeffding}
\end{align}
Here $E_1(t)$ is the event described in \Cref{lem:eliminatedOptimal}. If arm $\arm^*$ is not empirically non-competitive at round $\slot$, then arm $\arm$ can only be pulled in round $\slot + 1$ if $\Index_\arm(\slot) > \Index_{\arm^*}(\slot)$, due to which we have \eqref{eliminating1}. Inequalities \eqref{unionBound} and \eqref{usedHoeffding} follow from union bound and \Cref{lem:eliminatedOptimal} respectively.
% \eqref{eliminating1} follows from the fact that in order for arm $\arm$ to be pulled in round $\slot + 1$, we need $\Index_{\arm} > \Index_{\arm^*}$ if arm $\arm^*$ belongs to the set of empirically non-competitive arms. \\
% \eqref{unionBound} follows from the union bound. 
% \eqref{usedHoeffding} follows from the result of \Cref{lem:eliminatedOptimal}.

We now bound the second term in \eqref{usedHoeffding}.
\begin{align}
    &\Pr(\Index_\arm(\slot) > \Index_{\arm^*}(\slot) , \pulls_\arm(\slot) \geq s) \nonumber\\
    &= \Pr\left(\Index_\arm(\slot) > \Index_{\arm^*}(\slot) ,  \pulls_\arm(\slot) \geq s, \meanReward_{\arm^*} \leq \Index_{\arm^*}(\slot)\right)  + \nonumber \\
    &\quad \Pr\left(\Index_\arm(\slot) > \Index_{\arm^*}(\slot), \pulls_\arm(\slot) \geq s | \meanReward_{\arm^*} > \Index_{\arm^*}(\slot) \right) \times \Pr\left(\meanReward_{\arm^*} > \Index_{\arm^*}(\slot) \right) \label{conditionTerm} \\
    &\leq  \Pr\left(\Index_\arm(\slot) > \Index_{\arm^*}(\slot), \pulls_\arm(\slot) \geq s, \meanReward_{\arm^*} \leq \Index_{\arm^*}(\slot)\right) + \Pr\left(\meanReward_{\arm^*} > \Index_{\arm^*}(\slot)\right) \label{droppingTerms}\\
    &\leq \Pr\left(\Index_\arm(\slot) > \Index_{\arm^*}(\slot), \pulls_\arm(\slot) \geq s, \meanReward_{\arm^*} \leq \Index_{\arm^*}(\slot)\right) + \slot^{-3} \label{usingHoeffdingAgain}\\
    &= \Pr\left(\Index_\arm(\slot) > \meanReward_{\arm^*} ,  \pulls_\arm(\slot) \geq s\right) + \slot^{-4} \label{usingConditioning} \\
    &= \Pr\left(\hat{\meanReward}_\arm(\slot) + \sqrt{\frac{2 \log \slot}{\pulls_\arm(\slot)}} > \meanReward_{\arm^*} , \pulls_\arm(\slot) \geq s \right) + \slot^{-3} \label{expandingIndex}\\
    &= \Pr\left(\hat{\meanReward}_\arm(\slot) - \meanReward_\arm > \meanReward_{\arm^*} - \meanReward_\arm - \sqrt{\frac{2 \log \slot}{\pulls_\arm(\slot)}} , \pulls_\arm(\slot) \geq s \right) + \slot^{-3} \\ 
    &= \Pr\left( \frac{\sum_{\tau = 1}^{\slot} \indicator_{\{\arm_\tau = \arm\}}\reward_\tau}{\pulls_\arm(\slot)} - \meanReward_\arm > \gap_\arm - \sqrt{\frac{2 \log \slot}{\pulls_\arm(\slot)}} , \pulls_\arm(\slot) \geq s\right) + \slot^{-3} \\
    &\leq \slot \exp\left(-2 s \left(\gap_\arm - \sqrt{\frac{2 \log \slot}{s}}\right)^2\right) + \slot^{-3} \label{eqn:chernoffagain}\\
    &\leq \slot^{-3}\exp\left(-2 s \left(\gap_\arm^2 - 2 \gap_\arm \sqrt{\frac{2 \log \slot}{s}}\right)\right) + \slot^{-3} \\
    &\leq 2 \slot^{-3} \quad \text{ for all  } \slot > \slot_0. \label{finalCondn}
\end{align}
We have \eqref{conditionTerm} holds because of the fact that $P(A) = P(A|B)P(B) + P(A|B^c)P(B^c)$, Inequality \eqref{usingHoeffdingAgain} follows from \Cref{lem:ucbindexmore}. From the definition of $\Index_\arm(\slot)$ we have \eqref{expandingIndex}. Inequality \eqref{eqn:chernoffagain} follows from Hoeffding's inequality and the term $\slot$ before the exponent in \eqref{eqn:chernoff} arises as the random variable $\pulls_\arm(\slot)$ can take values from $s$ to $\slot$ (\Cref{lem:UnionBoundTrickInt}). Inequality \eqref{finalCondn} follows from the fact that $s > \frac{\slot}{2 \numArms}$ and $\gap_\arm \geq 4\sqrt{\frac{\numArms \log \slot_0}{\slot_0}}$ for some constant $\slot_0 > 0.$
% \eqref{expandingIndex} follows from the definition of $\Index_\arm(\slot)$. \\
% \eqref{eqn:chernoffagain} follows from Hoeffding's inequality, and\\
% \eqref{finalCondn} follows from the fact that $s > \frac{\slot}{2 \numArms}$ and $\gap_\arm \geq 4\sqrt{\frac{\numArms \log \slot_0}{\slot_0}}$ for some constant $\slot_0 > 0$.

Plugging this in the expression of $\Pr(\arm_\slot = \arm \mid \pulls_\arm (\slot) \geq s)$ \eqref{usedHoeffding} gives us, 
\begin{align}
    \Pr(\arm_{\slot+1} = \arm \mid \pulls_\arm (\slot) \geq s) &\leq \slot \exp\left(\frac{-\slot \gap_{\text{min}}^2}{2 \numArms}\right) + \Pr(\Index_\arm(\slot) > \Index_{\arm^*}(\slot) | \pulls_\arm(\slot) \geq s) \\
    &\leq \slot\exp\left(\frac{-\slot \gap_{\text{min}}^2}{2 \numArms}\right) + 2\slot^{-3} \\
    &\leq 3 \slot^{-3}. \label{usingConditiont0}
\end{align}
Here, \eqref{usingConditiont0} follows from the fact that $\gap_{\text{min}} \geq 2\sqrt{\frac{2\numArms \log \slot_0}{\slot_0}}$ for some constant $\slot_0 > 0$.  
\end{proof}

\begin{lem}
If for a suboptimal arm $\arm \neq \arm^*$, $\optimistGap_{\arm,\arm^*} > 0$, then,
$$\Pr(\arm_{\slot+1} = \arm, \pulls_{\arm^*}(\slot) = \max_\arm \pulls_\arm) \leq \slot \exp\left(\frac{-t\optimistGap_{\arm,\arm^*}^2}{2 \numArms}\right).$$

Moreover, if $\optimistGap_{\arm,\arm^*} \geq 2\sqrt{\frac{2 \numArms \log \slot_0}{\slot_0}}$ for some constant $\slot_0 > 0$. Then, 
$$\Pr(\arm_{\slot+1} = \arm , \pulls_{\arm^*}(\slot) = \max_\arm \pulls_\arm) \leq \slot^{-3} \quad \forall \slot > \slot_0.$$
\label{lem:suboptimalNotCompetitive}
\end{lem}

\begin{proof}
We now bound this probability as,
\begin{align}
    &\Pr(\arm_{\slot+1} = \arm , \pulls_{\arm^*} = \max_\arm \pulls_\arm) \nonumber \\
    &= \Pr\left(\hat{\meanReward}_{\arm^*}(\slot) < \estimateMean_{\arm,\arm^*}(\slot), \Index_\arm(\slot) = \max_{\arm'} \Index_{\arm'}(\slot) , \pulls_{\arm^*}(\slot) = \max_\arm \pulls_\arm(\slot) \right) \\
    &\leq \Pr\left(\hat{\meanReward}_{\arm^*}(\slot) < \estimateMean_{\arm,\arm^*}(\slot) , \pulls_{\arm^*}(\slot) = \max_\arm \pulls_\arm(\slot)\right) \\
    &\leq \Pr\left(\hat{\meanReward}_{\arm^*}(\slot) < \estimateMean_{\arm,\arm^*}(\slot) , \pulls_{\arm^*}(\slot) \geq \frac{\slot}{\numArms}\right) \\
    &\leq \Pr\left(\frac{\sum_{\tau = 1}^{\slot}\indicator_{\{\arm_\tau = \arm^*\}}\reward_\tau}{\pulls_{\arm^*}(\slot)} < \frac{\sum_{\tau = 1}^{\slot}\indicator_{\{\arm_\tau = \arm^*\}}\estimateReward_{\arm,\arm^*}(\reward_\tau)}{\pulls_{\arm^*(\slot)}} , \pulls_{\arm^*}(\slot) \geq \frac{\slot}{\numArms}\right)\\
    &= \Pr\left(\frac{\sum_{\tau = 1}^{\slot} \indicator_{\{\arm_\tau = \arm^*\}}(\reward_{\tau} - \estimateReward_{\arm,\arm^*})}{\pulls_{\arm^*}(\slot)} - (\meanReward_{\arm^*} - \expectedPseudoReward_{\arm,\arm^*}) < - \optimistGap_{\arm,\arm^*} , \pulls_{\arm^*} \geq \frac{\slot}{\numArms} \right) \label{eqn:lemchernoff}\\
    &\leq \slot \exp \left( \frac{- \slot \optimistGap_{\arm,\arm^*}^2}{2 \numArms} \right) \\
    &\leq \slot^{-3} \quad \forall \slot > \slot_0.\label{lastStepHere}
\end{align}

Here, \eqref{eqn:lemchernoff} follows from the Hoeffding's inequality as we note that rewards $\{\reward_\tau - \estimateReward_{\arm,\arm^*}(\reward_\tau): \tau=1,\ldots, t, ~ k_{\tau}=k\}$ form a collection of i.i.d. random variables each of which is bounded between $[-1,1]$ with mean $(\meanReward_\arm - \expectedPseudoReward_{\arm, \arm^*})$. The term $\slot$ before the exponent in \eqref{eqn:lemchernoff} arises as the random variable $\pulls_\arm(\slot)$ can take values from $\slot/\numArms$ to $\slot$ (\Cref{lem:UnionBoundTrickInt}). Step \eqref{lastStepHere} follows from the fact that $\optimistGap_{\arm,\arm^*} \geq 2\sqrt{\frac{2 \numArms \log \slot_0}{\slot_0}}$ for some constant $\slot_0 > 0$. 
\end{proof}

\begin{lem}
If $\gap_{\text{min}} \geq 4\sqrt{\frac{\numArms \log \slot_0}{\slot_0}}$ for some constant $\slot_0 > 0$, then, $$\Pr\left(\pulls_\arm(\slot) > \frac{\slot}{ \numArms}\right) \leq  3\numArms \left(\frac{\slot}{\numArms}\right)^{-2} \quad \forall \slot > \numArms \slot_0.$$
\label{lem:suboptimalNotPulled}
\end{lem}

\begin{proof}
We expand $\Pr\left(\pulls_\arm(\slot) > \frac{t}{\numArms}\right)$ as,
% \begin{align}
%     \Pr\left(\pulls_\arm(\slot) \geq \frac{\slot}{\numArms}\right) &= \Pr\left( \pulls_{\arm}(\slot) \geq \frac{\slot}{\numArms} \mid \pulls_\arm(\slot - 1) \geq \frac{\slot}{\numArms} \right) \Pr\left( \pulls_\arm(\slot - 1) \geq \frac{\slot}{\numArms} \right) \\
%     &+ \Pr\left(\arm_\slot = \arm \mid \pulls_\arm(\slot - 1) = \frac{\slot}{\numArms} - 1\right) \Pr \left(\pulls_\arm (\slot - 1) = \frac{\pulls}{\numArms} - 1\right) \\
%     &\leq \Pr\left(\pulls_\arm(\slot - 1) \geq \frac{\slot}{\numArms}\right) + \Pr\left(\arm_\slot = \arm \mid \pulls_\arm(\slot - 1) = \frac{\slot}{\numArms} - 1\right) \\
%     &\leq \Pr\left(\pulls_\arm(\slot - 1) \geq \frac{\slot}{\numArms}\right) + 3 (\slot - 1)^{-4} \\
%     &\leq \Pr\left(\pulls_\arm(\slot - 2) \geq \frac{\slot}{\numArms}\right) + \Pr\left(\arm_\slot = \arm \mid \pulls_\arm(\slot - 2) = \frac{\slot}{\numArms} - 1\right) + 3(\slot -1)^{-4} \\
%     &\leq \Pr\left(\pulls_\arm(\slot - 2) \geq \frac{\slot}{\numArms}\right) + 3(\slot - 2)^{-4} + 3(\slot - 1)^{-4} \\
%     &\leq \Pr\left(\pulls_\arm\left(\frac{\slot}{\numArms} - 1\right) \geq \frac{\slot}{\numArms}\right) + \sum_{\tau = \frac{\slot}{\numArms} - 1}^{\slot - 1} \tau^{-4} \\
%     &\leq \slot \times 3\left(\frac{\slot}{\numArms}\right)^{-4} \\
%     &= 3\numArms \left(\frac{\slot}{\numArms}\right)^{-3} \quad \forall \slot > \numArms \slot_0
% \end{align}

\begin{align}
    \Pr\left(\pulls_\arm(\slot) \geq \frac{\slot}{\numArms}\right) &= \Pr\left( \pulls_{\arm}(\slot) \geq \frac{\slot}{\numArms} \mid \pulls_\arm(\slot - 1) \geq \frac{\slot}{\numArms} \right) \Pr\left( \pulls_\arm(\slot - 1) \geq \frac{\slot}{\numArms} \right) + \nonumber \\
    &\quad \Pr\left(\arm_\slot = \arm , \pulls_\arm(\slot - 1) = \frac{\slot}{\numArms} - 1\right)  \\
    &\leq \Pr\left(\pulls_\arm(\slot - 1) \geq \frac{\slot}{\numArms}\right) + \Pr\left(\arm_\slot = \arm , \pulls_\arm(\slot - 1) = \frac{\slot}{\numArms} - 1\right) \\
    &\leq \Pr\left(\pulls_\arm(\slot - 1) \geq \frac{\slot}{\numArms}\right) + 3 (\slot - 1)^{-3} \quad \forall (\slot - 1) > \slot_0. \label{fromPrevLemma}
\end{align}
Here, \eqref{fromPrevLemma} follows from \Cref{lem:noMorePulls}.\\

This gives us $$\Pr\left(\pulls_\arm(\slot) \geq \frac{\slot}{\numArms}\right) - \Pr\left(\pulls_\arm(\slot - 1) \geq \frac{\slot}{\numArms}\right) \leq 3(\slot - 1)^{-3}, \quad \forall (\slot - 1) > \slot_0.$$
Now consider the summation $$ \sum_{\tau = \frac{\slot}{\numArms}}^{\slot} \Pr\left(\pulls_\arm(\tau) \geq \frac{\slot}{\numArms}\right) - \Pr\left(\pulls_\arm(\tau - 1) \geq \frac{\slot}{\numArms}\right) \leq \sum_{\tau = \frac{\slot}{\numArms}}^{\slot}3(\tau - 1)^{-3}.$$ This gives us, $$\Pr\left(\pulls_\arm(\slot) \geq \frac{\slot}{\numArms}\right) - \Pr\left(\pulls_\arm\left(\frac{\slot}{\numArms} - 1\right) \geq \frac{\slot}{\numArms}\right) \leq \sum_{\tau = \frac{\slot}{\numArms}}^{\slot}3(\tau - 1)^{-3}.$$
Since $\Pr\left(\pulls_\arm\left(\frac{\slot}{\numArms} - 1\right)\geq \frac{\slot}{\numArms}\right)  = 0$, we have, 
\begin{align}
    \Pr\left(\pulls_\arm(\slot) \geq \frac{\slot}{\numArms}\right) &\leq \sum_{\tau = \frac{\slot}{\numArms}}^{\slot}3(\tau - 1)^{-3} \\
    &\leq 3\numArms \left(\frac{\slot}{\numArms}\right)^{-2} \quad \forall \slot > \numArms \slot_0.
\end{align}

\end{proof}

\section{Proofs of Instance Dependent Bounds (Theorem 1,2,3)}

%\OY{Should each of these proofs be subsections?}

\textbf{Proof of Theorem 1}
We bound $\E{\pulls_\arm(\totalPulls)}$ as,
\begin{align}
&\E{\pulls_\arm(\totalPulls)} = \nonumber \\ 
&\E{\sum_{\slot = 1}^{\totalPulls}\indicator_{\{\arm_\slot = \arm\}}}\\
&= \sum_{\slot = 0}^{\totalPulls-1} \Pr(\arm_{\slot+1} = \arm) \\
&= \sum_{\slot = 1}^{\numArms \slot_0} \Pr(\arm_\slot = \arm) + \sum_{\slot = \numArms \slot_0}^{\totalPulls-1} \Pr(\arm_{\slot+1} = \arm) \\
&\leq \numArms \slot_0 + \sum_{\slot = \numArms \slot_0}^{\totalPulls-1}\Pr(\arm_{\slot+1} = \arm, \pulls_{\arm^*}(\slot) = \max_{\arm'} \pulls_{\arm'}(\slot))  + \nonumber\\
&\quad \sum_{\slot = \numArms \slot_0}^{\totalPulls-1} \sum_{\arm' \neq \arm^*} \Pr(\pulls_{\arm'}(\slot) = \max_{\arm''} \pulls_{\arm''}(\slot))\Pr(\arm_{\slot+1} = \arm |  \pulls_{\arm'}(\slot) = \max_{\arm''} \pulls_{\arm''}(\slot)) \\
&\leq \numArms \slot_0 + \sum_{\slot = \numArms \slot_0}^{\totalPulls-1} \Pr(\arm_{\slot+1} = \arm, \pulls_{\arm^*}(\slot) = \max_{\arm'} \pulls_{\arm'}(\slot)) + \nonumber \\
& \quad \sum_{\slot = \numArms \slot_0}^{\totalPulls-1} \sum_{\arm' \neq \arm^*} \Pr(\pulls_{\arm'}(\slot) = \max_{\arm''} \pulls_{\arm''}(\slot)) \\
&\leq \numArms \slot_0 + \sum_{\slot = \numArms \slot_0}^{\totalPulls - 1} \slot^{-3} + \sum_{\slot = \numArms \slot_0}^{\totalPulls} \sum_{\arm' \neq \arm^*} \Pr\left(\pulls_{\arm'}(\slot) \geq \frac{\slot}{\numArms}\right) \label{usingSomeLemma}\\
&\leq  \numArms \slot_0  + \sum_{\slot = 1}^{\totalPulls} \slot^{-3} + \numArms (\numArms - 1) \sum_{\slot = \numArms \slot_0}^{\totalPulls} 3 \left(\frac{\slot}{\numArms}\right)^{-2}. \label{usingSomeOtherLemma}
\end{align}
Here, \eqref{usingSomeLemma} follows from \Cref{lem:suboptimalNotCompetitive} and \eqref{usingSomeOtherLemma} follows from \Cref{lem:suboptimalNotPulled}.

\textbf{Proof of Theorem 2}

For any suboptimal arm $\arm \neq \arm^*$,
\begin{align}
    \E{\pulls_\arm(\totalPulls)} &\leq \sum_{\slot = 1}^{\totalPulls} \Pr(\arm_\slot = \arm) \\
    &\leq \sum_{\slot = 1}^{\totalPulls} \Pr(E_1(\slot) \cup (E_1^c(\slot), \Index_\arm > \Index_{\arm^*})) \label{beatOptimal} \\
    &\leq \sum_{\slot = 1}^{\totalPulls} \Pr(E_1(\slot)) + \Pr(E_1^c(\slot), \Index_\arm(\slot - 1) > \Index_{\arm^*}(\slot - 1)) 
\end{align}
\begin{align}    
    \E{\pulls_\arm(\totalPulls)} &\leq \sum_{\slot = 1}^{\totalPulls} \Pr(E_1(\slot)) + \Pr(E_1^c(\slot), \Index_\arm(\slot - 1) > \Index_{\arm^*}(\slot - 1)) \nonumber\\
    &\leq \sum_{\slot = 1}^{\totalPulls} \Pr(E_1(\slot)) + \Pr(\Index_\arm(\slot - 1)> \Index_{\arm^*}(\slot - 1)) \\
    &= \sum_{\slot = 1}^{\totalPulls} \slot \exp\left(- \frac{\slot \gap_{\text{min}}^2}{2 \numArms}\right) + \sum_{\slot = 0}^{\totalPulls-1} \Pr\left(\Index_\arm(\slot) > \Index_{\arm^*}(\slot)\right) \label{eliminatedArmProb} \\
    &= \sum_{\slot = 1}^{\totalPulls} \slot \exp\left(- \frac{\slot \gap_{\text{min}}^2}{2 \numArms}\right) + \E{\indicator_{\Index_\arm > \Index_{\arm^*}}(\totalPulls)} \label{followFromDefinition} \\
    &\leq 8 \frac{\log (\totalPulls)}{\gap_\arm^2} + \left(1 + \frac{\pi^2}{3}\right) + \sum_{\slot = 1}^{\totalPulls} \slot \exp\left(- \frac{ \slot \gap_{\text{min}}^2}{2 \numArms}\right). \label{fromAuer}
\end{align}
Here, \eqref{eliminatedArmProb} follows from \Cref{lem:eliminatedOptimal}. We have \eqref{followFromDefinition} from the definition of $\E{n_{\Index_\arm > \Index_{\arm^*}}(\totalPulls)}$ in \Cref{lem:AuerResult}, and \eqref{fromAuer} follows from \Cref{lem:AuerResult}.

\textbf{Proof of Theorem 3:} Follows directly by combining the results on Theorem 1 and Theorem 2.

\section{Lower bound proofs}
For these proofs we define $R_\arm = g_\arm(X)$ and $\tilde{R}_\arm = g_\arm(\tilde{X})$, where $f_X(x)$ is the probability density function of random variable $X$ and $f_{\tilde{X}}(x)$ is the probability density function of random variable $\tilde{X}$. 

\begin{lem}
If arm $k$ is competitive, i.e., $\optimistGap_{\arm, \arm^*} < 0$, then there exists $f_{\tilde{X}}(x)$ such that $\E{\tilde{R}_\arm} > \E{R_{\arm^*}}$ and $f_{\tilde{R}_{\arm^*}}(r) = f_{R_{\arm^*}}(r)$. 
\label{lem:ChangeDistribution}
\end{lem}

\begin{proof}
Informally the statement means that if there exists an arm $\arm$ such that \textit{Pseudo-Gap} of arm $\arm$ with respect to arm $\arm^*$ is less than $0$, then it is possible to change the distribution of random variable $X$ from $f_X(x)$ to $f_{\tilde{X}}(x)$ such that reward distribution of arm $\arm^*$ remains unchanged, but arm $\arm$ becomes better than $\arm^*$ in terms of expected reward.

We now prove this statement for the case when $X$ is a discrete random variable. A similar argument can be made to generalize the result for continuous $X$. If $\mathbb{P}_{X}$ is the original distribution of $X$, we show how to create a distribution $\mathbb{P}_{\tilde{X}}$ such that $\E{\tilde{R}_\arm} > \E{R_{\arm^*}}$ and $\mathbb{P}_{\tilde{R}_{\arm^*}}(r) = \mathbb{P}_{R_{\arm^*}}(r)$. Let $S(r) = \{x : g_{\arm^*}(x) = r \}$, the set of realizations $x$ for which $g_{\arm^*}(x) = r$. Define $$x(r) = \arg \max_{x \in S(r)} g_\arm(x).$$ 

Let $\mathcal{B}$ denote the set of values taken by $g_{\arm^*}(X)$, then for all $r \in \mathcal{B}$, we define $\mathbb{P}_{\tilde{X}}(x)$ as
$$\mathbb{P}_{\tilde{X}}(x) = 
\begin{cases}
(1 - \epsilon)\mathbb{P}_{R_{\arm^*}}(r) \quad \text{if } x = x(r), |S(r)| > 1. \\
\frac{\epsilon}{(|S(r)|-1)} \quad \text{if } x \in S(r), x \neq x(r), |S(r)| > 1 \\
\mathbb{P}_{R_{\arm^*}}(r) \quad \text{if } x = x(r), |S(r)| = 1.
\end{cases}
$$

% If $g_{\arm^*}(X)$ takes values in the set $\mathcal{B}$, then for all $r \in \mathcal{B}$, we define $\mathbb{P}_{\tilde{X}}(x) = (1 - \epsilon)\mathbb{P}_{R_{\arm^*}}(r)$ if $x = x(r)$ and $\delta$ if $x \in S(r), x \neq x(r)$. Here $\delta > 0$ is chosen such that $\sum_x \mathbb{P}_{\tilde{X}}(x) = 1.$
 Note that such a construction of $\mathbb{P}_{\tilde{X}}(x)$ does not change the reward distribution of Arm $\arm^*$. Moreover $\E{\tilde{R}_\arm} \geq (1 - \epsilon)\expectedPseudoReward_{\arm,\arm^*}$ (since rewards are always non-negative). Since $\optimistGap_{\arm,\arm*} < 0$ we can always choose $\epsilon > 0$ such that $(1 - \epsilon)\expectedPseudoReward_{\arm,\arm^*} - \E{R_{\arm^*}} > 0$ and subsequently, $\E{\tilde{R}_\arm} - \E{R_{\arm^*}}$ > 0.
\end{proof}

%\SG{This is the proof i was editing right now. Looks fine to me now}

% By the definition of $\optimistGap_{\arm, \arm^*}$, we have, 
% \begin{align}
%     \optimistGap_{\arm,\arm^*}  &= \int (\estimateReward_{\arm,\arm^*}(r) - r) f_{R_{\arm^*}}(r)dr \\
%     &= \int \estimateReward_{\arm, \arm^*}(r) f_{R_{\arm^*}}(r)dr - \int r f_{R_{\arm^*}}(r)dr \\
%     &= \int \estimateReward_{\arm, \arm^*}(r) f_{\tilde{R}_{\arm^*}}(r)dr - \int r f_{\tilde{R}_{\arm^*}}(r)dr \\
%     &> \delta
% % \end{align}
  
% \SG{Idea: shift almost all probability mass from all other points to the x that maximizes pseudo-reward of arm $\arm$, by keeping reward of arm $\arm^*$ same.}
% \SG{Since $\optimistGap_{\arm,\arm^*} > \delta$, we can create $f_{\tilde{X}}(x)$ such that $f_{\tilde{R}_\arm}(r) > 0$ if $f_{R_\arm}(r) > 0.$}

\textbf{Proof of Theorem 4}

Let arm $\arm$ be a $\textit{Competitive}$ sub-optimal arm, i.e $\optimistGap_{\arm,\arm^*} < 0$. Since $\optimistGap_{\arm,\arm^*} < 0$, From \Cref{lem:ChangeDistribution}, it is possible to change the distribution of $R_\arm$ such that $\E{\tilde{R}_\arm} > \E{R_{\arm^*}}$ and reward distribution of arm $\arm^*$ is unaffected, i.e  $f_{\tilde{R}_{\arm^*}}(r) = f_{R_{\arm^*}}(r)$. Moreover, by our construction of $f_{\tilde{R}_\arm}(r)$ in \Cref{lem:ChangeDistribution}, $D(f_{R_{\arm^*}}(r) || f_{\tilde{R}_{\arm}}(r)) < \infty$.

Therefore, if these are the only two arms in our problem, then from \Cref{lem:LaiRobbins2Arms}, $$\lim_{\totalPulls \rightarrow \infty}\inf \frac{\E{\pulls_\arm(\totalPulls)}}{\log \totalPulls} \geq \frac{1}{D(f_{R_\arm}(r) || f_{\tilde{R}_{\arm}}(r))}.$$

Moreover, if we have more $\numArms - 1$ sub-optimal arms, instead of just 1, then $$\lim_{\totalPulls \rightarrow \infty}\inf \frac{\E{\sum_{\ell \neq \arm^*} \pulls_{\ell}(\totalPulls)}}{\log \totalPulls} \geq \frac{1}{D(f_{R_{\arm}}(r)|| f_{\tilde{R}_{\arm}}(r))}.$$

Consequently, since $\E{\regret(\totalPulls)} = \sum_{ell = 1}^{\numArms} \gap_\ell \E{\pulls_{\ell}(\totalPulls)}$, we have 
\begin{align}
\lim_{\totalPulls \rightarrow \infty}\inf \frac{\E{\regret (\totalPulls)}}{\log (\totalPulls)} \geq \max_{\arm \in \setofArms}\frac{\Delta_k}{D(f_{R_k} || f_{\tilde{R}_k})}.
\end{align}

\section{Proof of Worst Case Regret Bound}
In this section, without loss of generality we assume that Arm 1 is optimal, and $\meanReward_1 > \meanReward_2 > \meanReward_3 > \meanReward_4 \ldots > \meanReward_\numArms.$ Correspondingly, we define the event $E_i(\slot)$ to denote that arm $i$ was empirically non-competitive in round $\slot+1$. Note that this notation is consistent with the definition of $E_1(\slot)$ in \Cref{lem:eliminatedOptimal}. 

\begin{lem}
$$\Pr(E_1(\slot), E_2(\slot) \ldots E_\ell(\slot)) \leq \exp\left(\frac{- \slot (\meanReward_{\ell + 1} - \meanReward_\ell)^2}{2 \numArms}\right),$$
Consequently, if $\meanReward_{\ell+1} - \meanReward_\ell \geq 3\sqrt{\frac{\numArms \log \totalPulls}{\totalPulls}}$, 
$$\Pr(E_1(\slot), E_2(\slot) \ldots E_\ell(\slot)) \leq \slot^{-2}.$$
\label{lem:intersectionProb}
\end{lem}

\begin{proof}
We expand $\Pr(E_1(\slot), E_2(\slot) \ldots E_\ell(\slot))$ as,
\begin{align}
    \Pr(E_1(\slot), E_2(\slot) \ldots E_\ell(\slot)) &= \Pr(E_\ell(\slot) \mid E_1(\slot), E_2(\slot) \ldots E_{\ell -1}(\slot)) \Pr(E_1(\slot), E_2(\slot) \ldots E_{\ell -1}(\slot)) \\
    &\leq \Pr(E_\ell(\slot) \mid E_1(\slot), E_2(\slot) \ldots E_{\ell -1}(\slot)) \\
    &\leq \sum_{\arm = \ell + 1}^{\numArms} \Pr(E_\ell(\slot) | \pulls_\arm(\slot) = \max_{\arm'}\pulls_{\arm'}(\slot))\Pr(\pulls_\arm(\slot) = \max_{\arm'}\pulls_{\arm'}(\slot)) \label{cantEliminateYourself} \\
    &\leq \max_{\arm \in \{\ell+1 \ldots \numArms\}} \Pr(E_\ell(\slot) , \pulls_\arm(\slot) = \max_{\arm'}\pulls_{\arm'}(\slot)) \\
    &\leq \slot \exp\left(\frac{- \slot (\meanReward_{\ell + 1} - \meanReward_\ell)^2}{2 \numArms}\right) \label{SeePrevAnalysis}\\
    &\leq \slot^{-2} \quad \text{if  } \meanReward_{\ell+1} - \meanReward_\ell \geq 3\sqrt{\frac{\numArms \log \totalPulls}{\totalPulls}}.
\end{align}
Here, \eqref{cantEliminateYourself} follows from the fact that arm $1, 2 \ldots \ell$ can all be empirically non-competitive with respect to arms $\ell+1, \ell+2 \ldots \numArms$ only. Analysis done in the proof of \Cref{lem:eliminatedOptimal} gives us \eqref{SeePrevAnalysis}. 
\end{proof}

\begin{lem}
If $\gap_\arm > \alpha \sqrt{\frac{\numArms \log \totalPulls}{\totalPulls}}$ for some $\alpha > 3\numArms$. Then there exists an arm $\ell$ ($\ell \leq \arm$) such that $\meanReward_{\ell} - \meanReward_{\ell - 1} \geq 3 \sqrt{\frac{\numArms \log \totalPulls}{\totalPulls}}$.
\label{lem:thereExistsAnArm}
\end{lem}
\begin{proof}
Since $\gap_\arm = \sum_{\arm' = 2}^{\arm} \meanReward_{\arm'} - \meanReward_{\arm'-1},$ it follows that $$\arm \left(\max_{\arm' = \{2, 3 \ldots \arm\}} \meanReward_{\arm'} - \meanReward_{\arm' - 1}\right) \geq \gap_\arm.$$
The statement of the lemma follows from the fact that $\gap_\arm > \alpha \sqrt{\frac{\numArms \log \totalPulls}{\totalPulls}}$, $\alpha > 3\numArms$ and $\arm < \numArms$. 
\end{proof}

\begin{lem}
If $\gap_\arm \geq \alpha \sqrt{\frac{\numArms \log \totalPulls}{\totalPulls}}$, and $\meanReward_{\ell} - \meanReward_{\ell-1} < 3 \sqrt{\frac{\numArms \log \totalPulls}{\totalPulls}}$ for all $\ell \leq \arm' \leq \arm$, then $$\meanReward_{\arm} - \meanReward_{\arm'} \geq \gamma \gap_\arm,$$
for some constant $0 < \gamma < 1$. 
\label{lem:gapStillSmall}
\end{lem} 

\begin{proof}
Expanding $\meanReward_\arm - \meanReward_{\arm'}$ gives us
\begin{align}
    \meanReward_\arm - \meanReward_{\arm'} &= \meanReward_\arm - \meanReward_1 - \sum_{\ell = 2}^{\arm'} (\meanReward_{\ell} - \meanReward_{\ell - 1}) \\
    &= \gap_\arm - \sum_{\ell = 2}^{\arm'} (\meanReward_{\ell} - \meanReward_{\ell - 1}) \\
    &= \gap_\arm \left(1 - \sum_{\ell = 2}^{\arm'} \frac{(\meanReward_{\ell} - \meanReward_{\ell - 1})}{\gap_\arm}\right) \\
    &\geq \gap_\arm \left(1 - \frac{3}{\alpha}\right) \label{pluginstuff} \\
    &= \gamma \gap_\arm.
\end{align}
Here, \eqref{pluginstuff} follows from the fact that $\gap_\arm \geq \alpha \sqrt{\frac{\numArms \log \totalPulls}{\totalPulls}}$. Since $\ell \leq \arm'$, we also have, $$\meanReward_\arm - \meanReward_\ell \geq \meanReward_\arm - \meanReward_{\arm'} \geq \gamma \gap_\arm \quad \forall \ell \leq \arm'.$$
\end{proof}

\begin{lem}
If $\gap_{\arm} > \alpha \sqrt{\frac{\numArms \log \totalPulls}{\totalPulls}}$ for some $\alpha > 3\numArms$, then
$$\E{\pulls_\arm(\totalPulls)} \leq \beta \frac{\log \totalPulls}{\gap_{\arm}^2}, \quad \text{for some } \beta > 0.$$ 
\label{lem:lastLem}
\end{lem}

 \begin{proof}

From \Cref{lem:thereExistsAnArm} there exists an arm $\ell$ ($\ell \leq \arm$) such that $\meanReward_{\ell} - \meanReward_{\ell - 1} \geq 3 \sqrt{\frac{\numArms \log \totalPulls}{\totalPulls}}$. Denote $\arm'$ to be the minimum $\ell$ such that $\meanReward_{\ell} - \meanReward_{\ell - 1} \geq 3\sqrt{\frac{\numArms \log \totalPulls}{\totalPulls}}$. Then we have,
\begin{align}
    &\E{\pulls_\arm(\totalPulls)} \leq \sum_{\slot = 1}^{\totalPulls} \Pr(\arm_\slot = \arm) \\
    &\leq \sum_{\slot = 1}^{\totalPulls} \Pr\big( (E_1^c(\slot), \Index_\arm > \Index_1) \bigcup (E_1(\slot), E_2^c(\slot), \Index_\arm > \Index_2) \bigcup \ldots \bigcup \nonumber \\
    &\quad ( E_1(\slot) E_2(\slot) \ldots E_{\arm - 1}^c(\slot), \Index_\arm > \Index_{\arm - 1}) \bigcup (E_1(\slot), E_2(\slot) \ldots E_{\arm - 1}(\slot))\big) \\
    &\leq \sum_{\slot = 1}^{\totalPulls} \Pr(\Index_\arm > \Index_1) + \Pr\big(E_1(\slot)\big)\Pr\big(\Index_\arm > \Index_2 | E_1(\slot)\big) + \nonumber \\
    & \quad \ldots \Pr\big(E_1(\slot), E_2(\slot), \ldots E_{\arm -2}(\slot)\big)\Pr\big(\Index_\arm > \Index_{\arm - 1} | E_1(\slot), E_2(\slot) \ldots E_{\arm - 2}(\slot)\big) + \nonumber \\
    &\quad \Pr\big(E_1(\slot), E_2(\slot) \ldots E_{\arm -1}(\slot)\big) \\
\end{align}
\begin{align}    
    &\leq \sum_{\ell = 1}^{\arm'} \sum_{\slot = 1}^{\totalPulls} \Pr\left(\Index_\arm > \Index_\ell \Bigg| \bigcap_{j = 1}^{\ell-1} E_j(\slot)\right)\Pr\left(\bigcap_{j = 1}^{\ell-1} E_j(\slot)\right) + \sum_{\ell = \arm'+1}^{\arm} \sum_{\slot = 1}^{\totalPulls} \Pr\left(\bigcap_{j = 1}^{\ell-1} E_j(\slot)\right)\\
    &\leq \sum_{\ell = 1}^{\arm'} \sum_{\slot = 1}^{\totalPulls} \Pr\left(\Index_\arm > \Index_\ell \right) + \sum_{\ell = \arm'+1}^{\arm} \sum_{\slot = 1}^{\totalPulls} \Pr\left(\bigcap_{j = 1}^{\arm'} E_j(\slot)\right) \\
    &\leq \sum_{\ell = 1}^{\arm'} 8\frac{\log \totalPulls}{(\meanReward_\arm - \meanReward_\ell)^2} + \left(1 + \frac{\pi^2}{3}\right) + \sum_{\ell = \arm'+1}^{\arm} \sum_{\slot = 1}^{\totalPulls} \Pr\left(\bigcap_{j = 1}^{\arm'} E_j(\slot)\right) \label{usedAuersResult}\\
    &\leq \sum_{\ell = 1}^{\arm'} 8\frac{\log \totalPulls}{(\gamma \gap_\arm)^2} + \left(1 + \frac{\pi^2}{3}\right) + \sum_{\ell = \arm'+1}^{\arm} \sum_{\slot = 1}^{\totalPulls} \Pr\left(\bigcap_{j = 1}^{\arm'} E_j(\slot)\right) \label{usedMylemma} \\
    &\leq \sum_{\ell = 1}^{\arm'} 8\frac{\log \totalPulls}{(\gamma \gap_\arm)^2} + \left(1 + \frac{\pi^2}{3}\right) + \sum_{\ell = \arm'+1}^{\arm} \sum_{\slot = 1}^{\totalPulls} \slot^{-2} \label{usedAnotherLemma}\\
    &\leq \numArms \left(8\frac{\log \totalPulls}{(\gamma \gap_\arm)^2} + \left(1 + \frac{\pi^2}{3}\right)\right) + \numArms \left(\sum_{\slot = 1}^{\totalPulls} \slot^{-2}\right) \\
    &\leq \numArms \left(8\frac{\log \totalPulls}{(\gamma \gap_\arm)^2} + \left(1 + \frac{\pi^2}{3}\right)\right) + \numArms \left(\left(1 + \frac{\pi^2}{3}\right)\right) \label{sumIsFinite}\\
    &\leq \beta \frac{\log \totalPulls}{\gap_\arm^2} \quad \text{for some } \beta > 0,
\end{align} 
where \eqref{usedAuersResult}  follows from \Cref{lem:AuerResult}. We have \eqref{usedMylemma} from \Cref{lem:gapStillSmall}. Inequality \eqref{usedAnotherLemma} follows from \Cref{lem:intersectionProb} and \eqref{sumIsFinite} follows from the fact that $\sum_{\slot = 1}^{\infty} \slot^{-2} = 1 + \frac{\pi^2}{3}.$
 \end{proof}

\textbf{Proof of Theorem 5}

From \Cref{lem:lastLem}, we have $\E{\pulls_\arm(\totalPulls)} > \frac{\beta \log (\totalPulls)}{\gap_\arm^2}$ if $\gap_\arm > \gap = 3 \numArms \sqrt{\frac{\numArms \log \totalPulls}{\totalPulls}}$ for some $\beta > 0$. Using this we can write,
\begin{align}
\E{\regret(\totalPulls)} &= \sum_{\arm \neq \arm^*} \gap_\arm \E{\pulls_\arm(\totalPulls)} \\
&= \sum_{\arm: \gap_\arm < \gap} \gap_\arm \E{\pulls_\arm(\totalPulls)} + \sum_{\arm: \gap_\arm > \gap} \gap_\arm \E{\pulls_\arm(\totalPulls)}  \\
&\leq \totalPulls \gap + \sum_{\arm: \gap_\arm > \gap} \beta \frac{\log(\totalPulls)}{\gap_\arm} \\
&\leq  3 \numArms \sqrt{ \numArms \totalPulls \log (\totalPulls)} + 3 \numArms \beta \sqrt{\frac{\totalPulls \log (\totalPulls) }{\numArms}} \\
&= \OO\left(\sqrt{\totalPulls \log(\totalPulls)}\right).
\end{align}